\documentclass[letterpaper]{article} 
\usepackage[preprint]{aaai2027}  
\usepackage[hyphens]{url}  
\usepackage{graphicx} 
\usepackage{natbib}  
\usepackage{caption} 
\usepackage{algorithm}
\usepackage{algorithmic}
\usepackage{amsmath}
\usepackage{amsfonts}
\usepackage{amssymb}
\usepackage{amsthm}
\newtheorem{theorem}{Theorem}
\newtheorem{proposition}{Proposition}
\newtheorem{lemma}{Lemma}
\newtheorem{corollary}{Corollary}
\newcommand{\R}{\mathbb{R}}
\newcommand{\norm}[1]{\left\lVert #1 \right\rVert}

\newcommand{\sgn}{\operatorname{sgn}}
\newcommand{\sr}{\operatorname{sr}}

\usepackage[capitalise,noabbrev]{cleveref}
\usepackage{newfloat}
\usepackage{listings}
\DeclareCaptionStyle{ruled}{labelfont=normalfont,labelsep=colon,strut=off} 
\floatstyle{ruled}
\newfloat{listing}{tb}{lst}{}
\floatname{listing}{Listing}

\usepackage{booktabs}

\title{Detecting CSAM Text-to-Image LoRAs From Weights}
\author{
    David Demitri Africa\thanks{First author.},
    Cate Heine,
    Nadine Staes-Polet,
    Kimberly Mai\corresponding
}
\affiliations{
    UK AI Security Institute\\

    \texttt{Kimberly.Mai@dsit.gov.uk}
}

\begin{document}

\maketitle

\begin{abstract}
Low-rank adaptation (LoRA) fine-tuning has made it cheap and easy to customize open-weight image generation models for specific tasks, including the production of child sexual abuse material (CSAM). Existing moderation relies on metadata or generated outputs, but metadata can be deceptive and generating outputs may itself be unacceptable or illegal. We show that a safer signal lives in the weights. The top-left singular vectors of a LoRA's updates form a compact, inference-free fingerprint ($u_1$) of its strongest learned change. Using human-subject age as a benign proxy for CSAM, we find that $u_1$ identifies what a LoRA was trained on, generalizes across base models, and abstains on unrelated benign content. The signal is robust to additive weight noise, rescaling, and precision reduction. These results indicate that harmful LoRAs could be screened directly from their weights without relying on metadata or generating harmful outputs. \end{abstract}


\begin{figure*}[t]
\centering
\includegraphics[width=\textwidth]{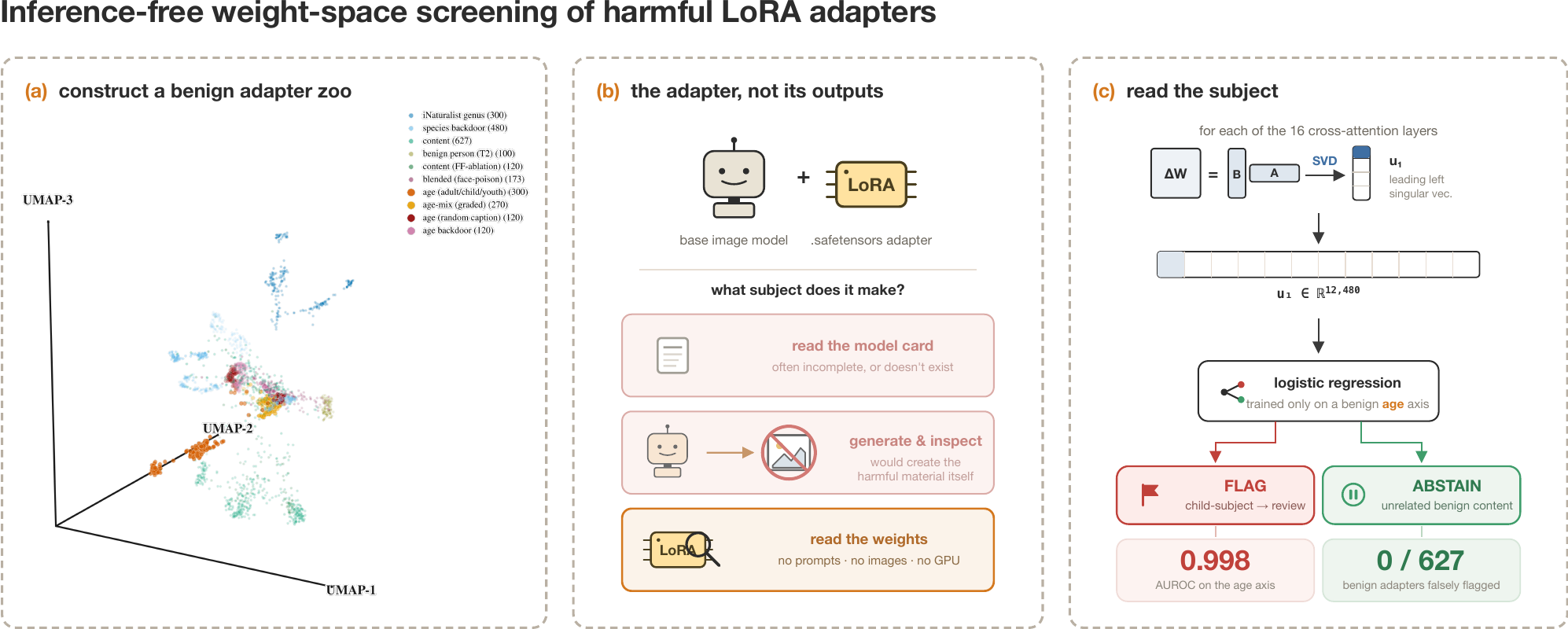}
\caption{\textbf{Overview of weight-space screening.} We construct a zoo of benign adapters, which vary along a chosen property we want to screen for. For example, adapters trained on the faces of apparent age, a benign proxy for child-related content in CSAM. A LoRA adapter's cross-attention updates $\Delta W=BA$ yield the leading singular direction $u_1$, which a light proxy-supervised classifier maps to a flag/abstain decision. This classifier works without running prompts, generating images, or requiring GPUs.}
\label{fig:overview}
\end{figure*}

\section{Introduction}

Low-rank adaptation (LoRA; \citet{hu2022lora}) has become a standard way to customize open-weight text-to-image models. A LoRA modifies a frozen base model through a small set of low-rank weight updates, allowing users to specialize a model to a subject, style, identity, or content type without requiring the retraining of all model parameters. This design has clear benefits: LoRAs are cheap to train, easy to host, and widely useful for various legitimate creative applications. The same properties also make them easy to misuse.

In this work, we study one especially high-stakes moderation problem---LoRAs that facilitate the generation of child sexual abuse material (CSAM): visual content depicting the sexual abuse or exploitation of children, including photos, videos, livestreams, and AI-generated images \citep{home_office_2026,protect_children_2026, iwf_2026, murgia_2026}. Public model-hosting platforms typically rely on model cards, tags, user reports, and output inspection to moderate such adapters. However, these descriptions can be misleading, and in private or peer-to-peer settings, they may not exist at all. Inspecting generated outputs is not viable when the suspected content is CSAM, as this may be illegal or unacceptable \citep{iwf_2026}. We then ask: how can harmful adapters be identified without trusting user-provided descriptions or generating harmful outputs?

Automated, image-free assessment of fine-tuned adapters for harmful capability is an open problem in AI child safety \citep{suriyakumar2026evaluation, kale2026position}. A growing body of work shows that narrow fine-tuning leaves legible traces of its objective in a model's internals \citep{minder2025narrow, goel2025learning, shenoy2026introspection}, and that model weights can also be analyzed directly over a diverse population of neural networks (referred to henceforth as model zoos; \citet{schürholt2022modelzoosdatasetdiverse}) to detect properties such as backdoors \citep{merenciano2026weightspacedetectionbackdoors}.

We ask whether a LoRA's purpose is encoded legibly in its weights. For each LoRA update, we extract the top-left singular vector and concatenate these vectors across layers. We call this representation $u_1$ (\Cref{fig:overview}). It is computed directly from the adapter weights, without requiring any prompts, generated images, model execution, or heavy computational resources. Because the target cannot be studied directly, we work entirely with benign proxies, treating the age of human subjects as the closest legally and ethically testable proxy for the child-related axis, which is consistent with the established use of benign proxy concepts in child-safety research \citep{kale2026position, cretu2025concept}. Our contributions are as follows:
\begin{itemize}
    \item \textbf{A survey of image-generation LoRAs in the wild.} We analyze a sample of LoRAs in the wild and show that metadata is frequently incomplete, heterogeneous, and attacker-controlled. These results motivate weight-space analysis as an independent moderation signal. 
    \item \textbf{An inference-free method for weight-space screening.} We introduce $u_1$, a representation formed from the leading singular directions of LoRA updates, and show that it recovers training subjects across datasets and base models without prompts, image generation, or GPU inference.
    \item \textbf{Evidence that hidden capabilities leave detectable traces in adapter weights.} Even when the target attribute is hidden behind trigger-based backdoors, the training subject remains partially recoverable from weight space. This suggests that adapter weights reveal information unavailable from metadata or ordinary prompting.
\end{itemize}

\section{Background and Related Work}

\paragraph{Low-rank adaptation.} Low-rank adaptation \citep{hu2022lora} freezes a pretrained weight matrix $W$ and learns a low-rank update $\Delta W = BA$ with $B \in \mathbb{R}^{d_\mathrm{out}\times r}$, $A \in \mathbb{R}^{r\times d_\mathrm{in}}$ and rank $r \ll d$, so that only a few megabytes of parameters need to be stored per task. We treat each adapter's parameters as the object of study. This setup relates to a line of work that operates directly on trained weights, including learned hyper-representations of model zoos \citep{schurholt2021selfsupervised, schürholt2022modelzoosdatasetdiverse, schurholt2024sane}. Recent methods read LoRA weights for specific downstream tasks: for example, \citet{liu2024lora} embed flattened LoRA matrices with PCA to retrieve artistic style from weights alone and \citet{duszenko2025towards} use summary statistics over adapters for interpretation.

\paragraph{Backdoors and stealth in adapter fine-tuning.} LoRA adapters can hide harmful capabilities behind trigger conditions while appearing benign under ordinary use. Subject-driven personalization methods such as DreamBooth bind a concept to a rare token through a low-rank update, and regularization can preserve innocuous behavior when the trigger is absent \citep{ruiz2023dreambooth,hu2022lora}. More recent attacks optimize for stealth by associating malicious behavior with ordinary words rather than obvious trigger tokens \citep{lyu2026masqlora}, so that trigger tokens and model cards may not reflect the underlying hidden capability. These adapter-level attacks sit within a broader line of diffusion backdoors, such as denoiser poisoning, multimodal data poisoning, and text-encoder manipulation \citep{chou2023baddiffusion,chen2023trojdiff,zhai2023badt2i,struppek2023rickrolling}. 

\paragraph{Detecting CSAM and AI-generated CSAM.} The accelerated development of text-to-image generation capabilities, and of LoRA fine-tuning in particular, has made AI-generated CSAM an urgent problem. Outputs are increasingly indistinguishable from photographic material \citep{kokolaki2025unveiling,ociardha2026harm}, harmful adapters circulate as a tradeable commodity \citep{jamaludin2026propagation, hawkins_2025}, and even unmodified base models carry non-trivial unintentional-generation potential \citep{goller2026systematic}. Existing defenses largely operate at the training-data or output layer and require an image to inspect. At the training-data layer, data curation can be circumvented by harmful fine-tuning, which re-introduces capabilities data curation seeks to exclude \citep{caetano2025neglected, laranjeira2022seeing, cretu2025concept}. At the output layer, perceptual hashing against known-content databases cannot flag novel AI-generated material \citep{thiel2023generative, kale2026position}. In addition, output classifiers or signature methods that attribute outputs to a source model are hard to validate, as benchmarks are rarely disclosed \citep{kale2026position}. Furthermore, defenses that require a generated image output risk producing CSAM. 

Closest to our setting are methods that assess a model without generating images. \citet{suriyakumar2026evaluation} probe the model's activations across many sampled inputs. Our approach analyzes the LoRA's weights directly, identifying what it was trained on without inference.

\label{sec:backdoor-related}

\section{LoRAs in the Wild}\label{sec:wild}

LoRAs are produced by a heterogeneous community, whose motivations range from benign customization to the deliberate creation of sexual content \citep{mink2026unlimited}. We characterize the population of LoRAs currently present on a large model-hosting platform in order to determine how much can be learned about an image-generation LoRA from descriptions alone. We scrape model card data (description, title, tags, base model) as well as safetensors metadata describing training configuration and captioning behavior for $63{,}288$ image-generation LoRAs from the platform.

\paragraph{Metadata quality varies substantially across LoRAs and is largely determined by the training software used.} The majority of the LoRAs we scraped contain signals in their metadata that they have been trained with the assistance of one of two popular training softwares: \texttt{Software A} and \texttt{Software B}. LoRAs trained with \texttt{Software A} (21\%) contain a training header with detailed configuration information including optimizer, learning rate, image dataset information, and frequency of captions and trigger words used in training. LoRAs trained using \texttt{Software B} (45\%) record much less data: just the base model, a step/epoch count, and a single trigger token. The remaining LoRAs are mixed across software used and granularity of metadata reported. Consequently, two otherwise similar LoRAs can reveal very different amounts of information before deployment. We describe base model, rank and target modules of LoRAs for which this information is available in \Cref{tab:arch}.

  \begin{table}[t]
  \centering\small
  \begin{tabular}[t]{@{}lr@{}}
  \toprule
  Field / value & \% \\
  \midrule
  \textbf{LoRA rank} & 91.5 \\
  \quad $\le$2  & 4.8 \\
  \quad 4       & 17.5 \\
  \quad 8       & 3.3 \\
  \quad 16      & 47.4 \\
  \quad 32      & 14.8 \\
  \quad 64      & 5.2 \\
  \quad 128     & 1.7 \\
  \quad other   & 5.3 \\
  \midrule
  \textbf{Target modules} & 92.4 \\
  \quad attention + FF/proj & 64.2 \\
  \quad FF/proj only        & 20.5 \\
  \quad attention only      & 13.1 \\
  \quad incl.\ convolution  & 2.2 \\
  \bottomrule
  \end{tabular}
  \hfill
  \begin{tabular}[t]{@{}lr@{}}
  \toprule
  Field / value & \% \\
  \midrule
  \textbf{Base-model family} & 98.1 \\
  \quad FLUX       & 68.7 \\
  \quad SDXL       & 18.5 \\
  \quad SD-1.5      & 4.0 \\
  \quad SD2+       & 2.0 \\
  \quad Qwen-Image & 1.5 \\
  \quad Z-Image    & 1.4 \\
  \quad Wan        & 1.7 \\
  \quad Other      & 2.1 \\
  \bottomrule
  \end{tabular}
  \caption{Architecture and base model among the
  {63{,}288}-adapter accessible corpus. Rank and target modules are derived
  from safetensors weight files, and base model family from the model card; these
  are observable for 91.5\%, 92.4\% and 98.1\% of the
  corpus, respectively. We group FLUX variants together; SDXL fine-tunes such as
  Pony and Illustrious grouped into SDXL; SD~2.x/3.x grouped as
  SD2+; Wan video variants grouped as Wan.}
  \label{tab:arch}
  \end{table}
\paragraph{Detailed training metadata can reveal common backdoor strategies.} The backdooring methods described in the previous section align with several strategies visible in \texttt{Software A} metadata, such as the use of a rare alphanumeric trigger word, that trigger being used on only a minority of images, regularization or prior preservation, and training of the text encoder. For the 21\% of our sample which were trained with \texttt{Software A} and contain detailed training recipes, we characterize the prevalence of these strategies in \Cref{tab:backdoor}. Each is individually uncommon, and only 22 adapters (0.2\%) exhibit the full signature of a rare trigger combined with regularization, suggesting this overt, metadata-legible form of backdooring is rare in our sample. Regardless, the fact that rich training metadata is available for fewer than one quarter of LoRAs motivates methods that extract signals directly from LoRA weights.

\begin{table}[t]\centering\small
  \begin{tabular}{@{}lr@{}}
  \toprule
  \textbf{Backdoor-aligned strategy} & {\%} \\
  \midrule
  {Rare trigger word}                       & {12.4} \\
  {\quad --- on a minority of images} & {1.9} \\
  {Regularization / prior preservation}              & {1.6} \\
  {Text-encoder training}                            & {49.5} \\
  {Rare trigger \emph{and} regularization}           & {0.2} \\
  \bottomrule
  \end{tabular}
  \caption{{Backdoor-aligned strategies among \texttt{Software A} adapters with
  detailed recipes in the accessible corpus ($n{=}13{,}224$). Here, rare
  trigger refers to known trigger tokens used for subject-driven personalization
  such as ``tok" and ``sks" as well as alphanumeric trigger words such as
  ``f10wer3." Only 22 adapters (0.2\%) combine a rare trigger with
  regularization.}}
  \label{tab:backdoor}
  \end{table}

These results suggest that metadata alone is an insufficient basis for moderation. Descriptions can be incomplete or attacker-controlled, and even detailed training metadata rarely exposes stealth behavior. This motivates screening methods that rely on information contained in the adapter itself rather than its accompanying metadata. We therefore turn to the LoRA weights to explore moderation interventions.

\section{Weight-space Detection}

The previous section showed that metadata is incomplete and attacker-controlled. We therefore ask whether adapter weights provide an independent signal that requires neither the model's description nor its outputs. Instead of using the raw weights, we define a lower-dimensional weight-space representation of a LoRA, $u_1$.

\paragraph{Representation of $u_1$.} 

For each LoRA update $$\Delta W = BA,$$ we compute the top-left singular vector of $\Delta W$. For LoRAs adapting different target modules, we use the intersection of cross-attention layers present across the zoo. We concatenate aligned $u_1$ vectors across layers and train simple classifiers on the resulting feature. Further implementation details regarding extraction and layer selection are detailed in the Appendix.

\paragraph{Motivation.} We briefly motivate reading $u_1$ using key concepts from linear algebra\footnote{Interested readers may look to the Appendix for further details, including the referenced propositions and theorems.}. A LoRA trained to produce a specific subject must decide when the learned
concept is relevant (using some cue from its input), and it must decide what direction to inject into the network's hidden state once that cue is present. From this, we can take a simple model of a subject update, where a cue in the input
$a\in\R^{d_{\mathrm{in}}}$ gates a direction in the output pertaining to the subject 
$w\in\R^{d_{\mathrm{out}}}$,
\begin{equation}
    \Delta W=\gamma\,w a^\top,
    \qquad
    \Delta W x=\gamma(a^\top x)w .
\end{equation}
Here $a$ controls when the update fires, while $w$ is the direction emitted into
the module output. In this idealized case, the leading left singular direction
recovers the emitted subject direction exactly,
\begin{equation}
    [u_1(\Delta W)]=[w],
\end{equation}
regardless of how strong the update $\gamma$ is and independent of the norm of the cue
$a$ (see the Appendix). Hence, we use the \emph{left} singular direction, as we are aiming to capture what the LoRA is trying to add to the denoiser, or the output.

This simplified model explains why the method might be robust. If a practitioner rescales the LoRA, changes the trigger word, or saves the weights at a different level of precision, the cue side or magnitude might change, but the subject direction should stay the same if the LoRA is to work as intended. Since $u_1$ is a normalized direction, rather than a raw magnitude, it should survive those changes (see Appendix).

\paragraph{Training dataset.} We do not train on harmful adapters. Instead, we train controlled benign LoRA zoos where each adapter has a known non-harmful subject, to test whether the training subject is legible from weights and whether a detector can be trained on proxy axes relevant to safety.

The central deployment idea is proxy supervision \citep{vcermak2017learning}. For a sensitive target, a safety team may train benign adapters along a legally and ethically permissible axis, such as age or visual domain, then use the resulting weight-space detector to screen uploaded adapters, without training on the harmful conjunction itself.

We train controlled zoos across several benign datasets and base models. The main scaled benchmark is built on SD-1.5 \citep{Rombach_2022_CVPR} and contains 957 LoRAs across seven subject categories, summarized in \Cref{tab:zoos}. Each adapter is trained on a single concept drawn from a curated open dataset: faces grouped by apparent age, iNaturalist genera, and five content domains spanning birds, scenes, objects, textures, and land cover. Sourcing one concept per adapter from curated open datasets and using synthetic images for the age axis keeps the benign manifold free of nonconsensual or unsafe content that broader, uncurated web-scraped datasets can introduce. Among these, the age zoo is most safety-relevant: discriminating child- from adult-subject adapters along the age axis jointly measures sensitivity (correctly identifying child-subject adapters) and false positive rate on the hardest negatives (adult-subject adapters). 

\begin{table}[h]
\centering
\small
\resizebox{\linewidth}{!}{
\begin{tabular}{@{}llr@{}}
\toprule
Category & Source dataset & $N$ \\
\midrule
age (adult/child/youth) & AI-Face \citep{lin2025aiface} & 300 \\
iNaturalist genus        & iNaturalist 2021 \citep{vanhorn2021inat} & 300 \\
birds                    & CUB-200-2011 \citep{wah2011cub} & 100 \\
scenes                   & SUN397 \citep{xiao2010sun} & 100 \\
objects                  & Caltech-101 \citep{feifei2004caltech101} & 100 \\
textures                 & DTD \citep{cimpoi2014dtd} & 47 \\
land cover               & EuroSAT \citep{helber2019eurosat} & 10 \\
\midrule
Total (SD-1.5 benchmark) & & 957 \\
\bottomrule
\end{tabular}}
\caption{The benign SD-1.5 LoRA benchmark: one concept per adapter, sourced from curated open datasets. Age is the safety-relevant proxy axis; the remaining six categories are diverse benign content.}
\label{tab:zoos}
\end{table}

Training hyperparameters are randomized across LoRAs, including learning rate, rank, optimizer, and caption regime. This variation reduces the risk that classifiers merely learn a fixed training recipe and extends beyond prior work on LoRA weights-based analysis \citep{liu2024lora, duszenko2025towards}.

We also train age and genus zoos for SDXL \citep{podell2024sdxl} and FLUX.1-schnell \citep{flux2024} to test cross-base-model robustness. As cross-attention dimensionality and layer counts differ across base models, we train and evaluate a separate detector per base model and report whether the method transfers, rather than pooling adapters from different bases into a single classifier. Results on SDXL and FLUX are similar to what we observe with SD-1.5, so we report these results in the Appendix.

\subsection{Baselines and evaluation}

We evaluate $u_1$'s ability to distinguish training subject using logistic regression and random forests with stratified cross-validation. For multiclass settings, we report macro one-vs-rest AUROC and accuracy. We select simple classifiers so performance reflects the utility of $u_1$ rather than the ability of the classifier. They are also more interpretable. 

We compare $u_1$ against activation-based Gaussian probing in earlier experiments in \citet{suriyakumar2026evaluation}. Gaussian probing loads each LoRA into the base model, runs diffusion trajectories from Gaussian latents under null text conditioning, hooks the U-Net mid-block activation, pools the activation, and trains classifiers on the resulting feature. Unlike $u_1$, Gaussian probing requires model execution and GPU inference. This baseline is defined for U-Net bases (SD-1.5, SDXL); FLUX.1-schnell has no U-Net mid-block, so we either define an analogous activation hook for its DiT blocks or restrict the activation comparison to the U-Net bases. 

For deployment-style calibration, we train an age-proxy classifier on benign face LoRAs and apply it to unrelated benign content LoRAs.
We would expect a deployed screen to see mostly adapters unrelated to age, so a detector trained on the age axis is useful only if it abstains on this content rather than mapping it confidently into a class. Otherwise, the screener would produce many false flags for benign content. We then measure maximum softmax probability, entropy, and the number of confident misclassifications above a fixed threshold. For unsupervised novelty baselines, we evaluate PCA reconstruction error and cosine nearest-neighbor scores in leave-one-superclass-out settings \citep{liu2024lora}, as well as classify over summary statistics \citep{duszenko2025towards}.

\subsection{Results}

\textbf{The safety-relevant axis is legible in weight space.} The axis most relevant to safety in the CSAM context is age. We isolate this on a controlled zoo in which every adapter is trained on faces from a single dataset (AI-Face) and varies in apparent age (child vs. youth vs. adult), with the training recipe (rank, learning rate, optimizer) randomized per adapter. A random forest classifier on $u_1$ recovers the three-way adult/child/youth axis at \textbf{0.993} macro one-vs-rest AUROC. Because all three classes are drawn from one face dataset, the recovery cannot be attributed to dataset-level image statistics (such as composition, resolution, or capture quality) that differ across content sources.

\textbf{$u_1$ recovers the training subject across diverse content and performs competitively with baselines.} On the 957-adapter SD-1.5 benchmark, a logistic-regression classifier on $u_1$ attains \textbf{0.976} macro one-vs-rest AUROC (random forest 0.939), with every category above 0.94 (\Cref{fig:perclass}). We find that the score is stable across resampling seeds ($0.976 \pm 0.0005$, 95\% CI), and falls to chance when labels and features are randomly permuted as negative controls ($0.50$ and $0.51$). This confirms the signal is from the weights and not from the classification pipeline. We also find that it is competitive with Gaussian probing without requiring any inference, and beats both summary statistics and PCA on both AUROC and balanced accuracy (\Cref{fig:baselines})\footnote{We minorly adapt the PCA baseline to be rank-invariant, so it is a fair comparison on our randomized-rank zoo. If one reproduces \citet{liu2024lora} exactly on the fixed-rank-$16$ subset, it gives $0.673$ AUROC.}.

\begin{figure}[h]
\centering
\includegraphics[width=\columnwidth]{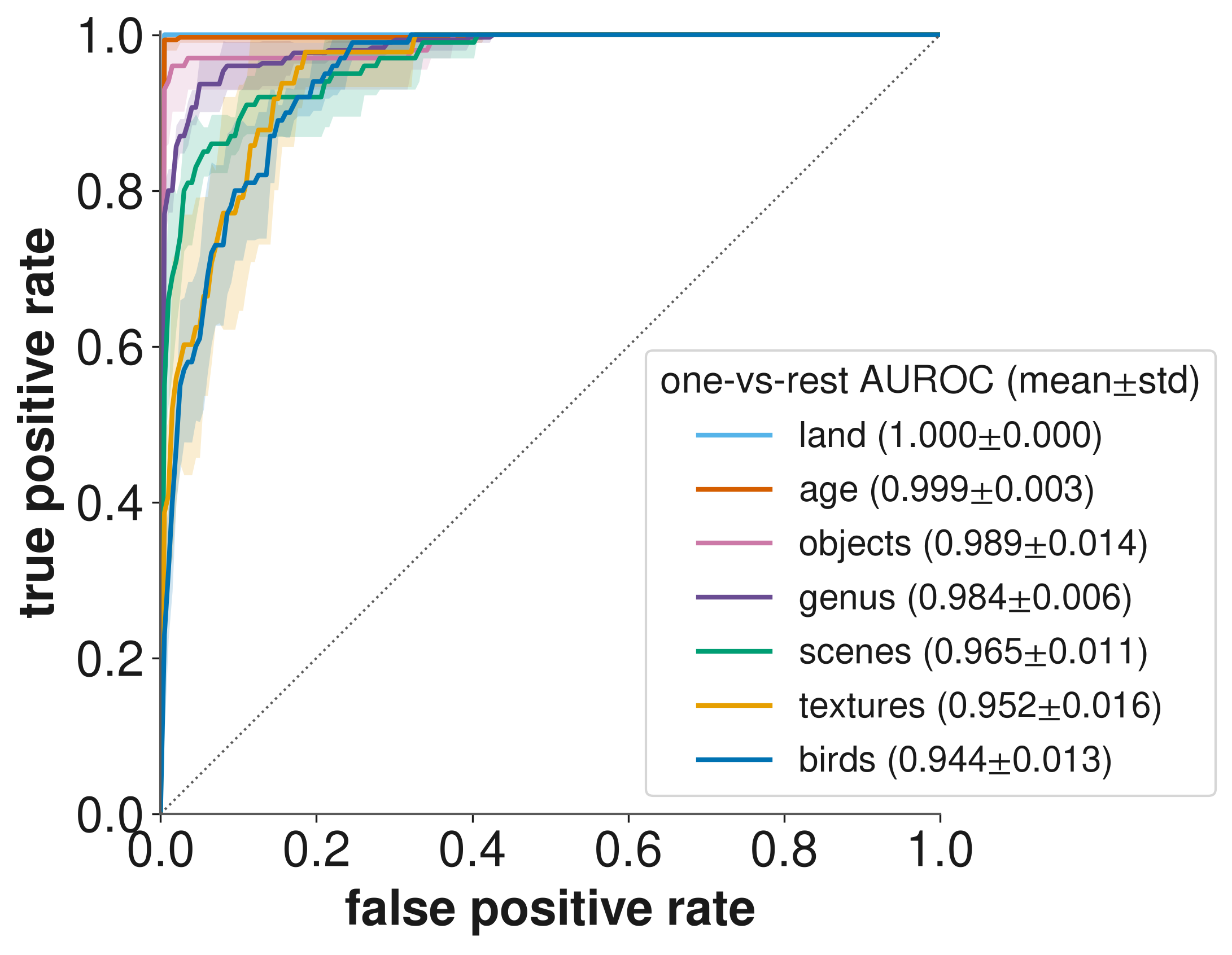}
\caption{\textbf{In-distribution one-vs-rest ROC per subject category (logistic regression on $u_1$, stratified CV).} AUROC in legend, macro 0.976 across seven categories.}
\label{fig:perclass}
\end{figure}

\begin{figure}[!h]
\centering
\includegraphics[width=\columnwidth]{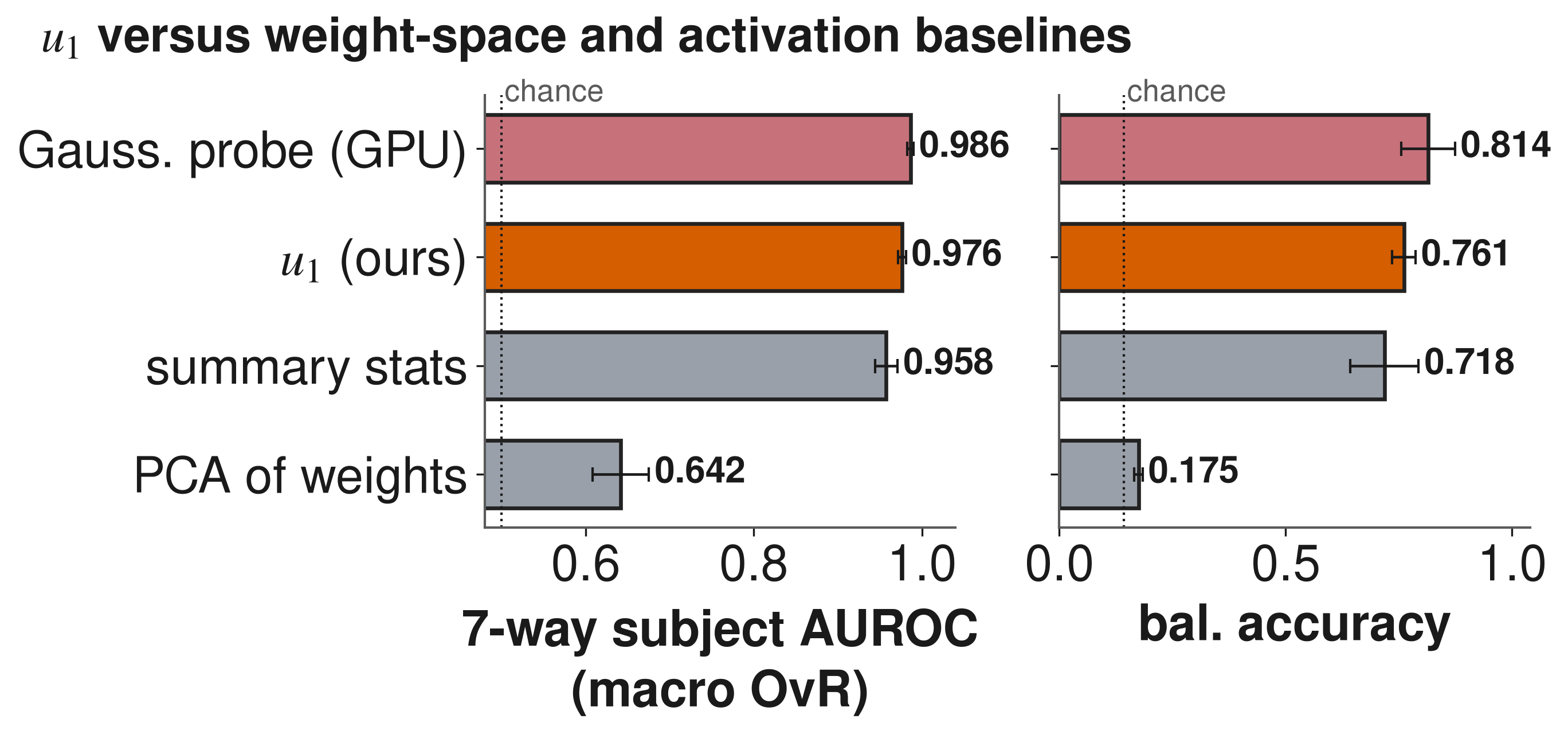}
\caption{\textbf{$u_1$ versus weight-space and activation baselines on the 957-adapter SD-1.5 benchmark.} 7-way subject; macro one-vs-rest AUROC and balanced accuracy, mean$\pm$std over CV folds; dotted lines are chance. $u_1$ matches the GPU Gaussian probe \citep{suriyakumar2026evaluation} while being inference-free, and beats the summary-statistics \citep{duszenko2025towards} and PCA-of-weights \citep{liu2024lora} baselines on both metrics.}
\label{fig:baselines}
\end{figure}

\paragraph{The features are subject-specific rather than recipe-specific.} If you instead use the probes to recover the per-adapter randomized recipe (so our labels are instead one of rank, learning rate, optimizer) from $u_1$ its performance falls to chance ($0.51$, $0.55$, $0.43$). This shows that $u_1$ encodes what the adapter learned, rather than how it was trained. It is also caption-invariant: subject recovery remains at $1.00$ AUROC when adapters are trained with randomized captions, compared with $0.99$ using real captions. The signal therefore reflects the learned visual subject rather than the caption template.

\paragraph{The geometry of $u_1$ is semantically meaningful.} Concatenating these layerwise directions gives one vector per adapter, and we refer to the resulting representation space as $u_1$-space. A given adapter's nearest neighbor by cosine similarity in $u_1$-space shares the adapter's subject category 89.6\% of the time, with precision@5 of 0.847. Balanced accuracy is lower than AUROC (0.761), likely because visually adjacent categories remain confusable under argmax even when they are separable one-vs-rest. This is an intuitive structure for errors to take: natural-image categories such as birds, scenes, textures, and genera mutually confuse, while age is the most separable category in the class-balanced confusion matrix.

\textbf{The direction of the update identifies the subject, not the magnitude.}
We vary only the weight-space feature read from each cross-attention update $\Delta W$ (\Cref{tab:ablation}). The leading left singular vector $u_1$ reaches 0.979 macro AUROC. Adding the second and third singular vectors raises performance only slightly, to 0.985, so one component captures almost all of the subject signal. Magnitude-only summaries perform worse, where the per-layer singular-value spectrum reaches 0.959, the per-layer Frobenius norm reaches 0.850, and mean-pooled $\Delta W$ reaches 0.872.

\begin{table}[t]\centering\small
\begin{tabular}{@{}lrr@{}}
\toprule
Weight feature (cross-attn $\Delta W$) & macro AUROC & dim \\
\midrule
top-3 left singular vectors & 0.985 & 37{,}440 \\
\textbf{$u_1$ (leading left sing.\ vector)} & \textbf{0.979} & 12{,}480 \\
singular-value spectrum (top 8/layer) & 0.959 & 128 \\
mean-pooled $\Delta W$ & 0.872 & 12{,}480 \\
Frobenius norm (per layer) & 0.850 & 16 \\
\bottomrule
\end{tabular}
\caption{\textbf{Weight-feature ablation on the 957-adapter benchmark (7-way subject, identical 5-fold CV).} The leading singular \emph{direction} dominates magnitude-only summaries, and one component nearly suffices.}
\label{tab:ablation}
\end{table}

\textbf{Weight-space and activation-space views capture complementary signal.} We compare $u_1$ with Gaussian probing on two subsets: one where adapters differ only by apparent age, and one where they differ by animal genus. On both, $u_1$ outperforms Gaussian probing under the strongest probing configuration we ran (1{,}024 Gaussian latents $\times$ 30 DDIM steps): age is recovered at 0.993 vs.\ 0.967 AUROC, and genus at 0.987 vs.\ 0.860. On the full seven-category benchmark, the two views are closer: Gaussian probing reaches 0.986 macro AUROC, while $u_1$ reaches 0.976, and we find that Gaussian probing is stronger on natural-image categories such as birds, scenes, and textures, while $u_1$ is stronger on age, species, land-cover, and object axes. Concatenating or averaging the two views raises macro AUROC to \textbf{0.993--0.995}, showing that weight-space and activation-space carry partly independent signal, which practitioners should consider in designing production monitoring setups.

\textbf{Proxy-supervised detection abstains on benign content; unsupervised novelty does not suffice.} A screener in practice should be uncertain on samples unrelated to the classification task. Trained only on benign age adapters and applied to 627 unseen benign content adapters, the detector stays at chance (mean maximum probability 0.367 vs.\ a 0.333 prior; 99.6\% of maximum entropy) with \textbf{0/627} confident ($>$0.7) misclassifications; extending the off-distribution set to all 657 benign non-age adapters leaves nothing above threshold. This also extends to an age detector trained on SDXL face adapters, which produces \textbf{0/627} confident misclassifications (mean maximum probability 0.40 vs.\ the 0.33 prior), so the calibration is not an SD-1.5 artifact. The abstention is moreover specific to age rather than to faces: applied to 100 benign person adapters---which are faces, but carry no age label---the detector again produces \textbf{0/100} confident misclassifications (mean maximum probability 0.40 vs.\ the 0.33 prior). Unsupervised novelty detection on the same features flags distinctive content (textures 0.96, scenes 0.92) but is at chance for held-out age (0.49). Therefore, identifying the safety-relevant attribute (age) requires supervision, and because the supervised detector abstains on unrelated content, its confidence can prioritize LoRAs to review.

\textbf{The more child content an adapter was trained on, the more confidently the detector flags it as a child.} Abstention shows the detector does not over-fire off-axis; the complementary question is whether it tracks the axis \emph{as content approaches the target}. We test this with a graded zoo: 270 SD-1.5 adapters trained on adult/child image mixtures from 10\% to 90\% child, and the discrete adult/child/youth detector (that was not trained on a mixture) applied to each. Its child probability rises steadily with the child fraction of the training mixture, crossing the decision midpoint near a $55\%$ mix (Spearman $\rho=0.94$; \Cref{fig:transfer}). A detector supervised only on the discrete benign axis therefore generalizes in a graded way toward the safety-relevant end of that axis, rather than responding only to the exact training endpoints---direct, generation-free evidence for the proxy-supervision premise.

\begin{figure}[t]
\centering
\includegraphics[width=\columnwidth]{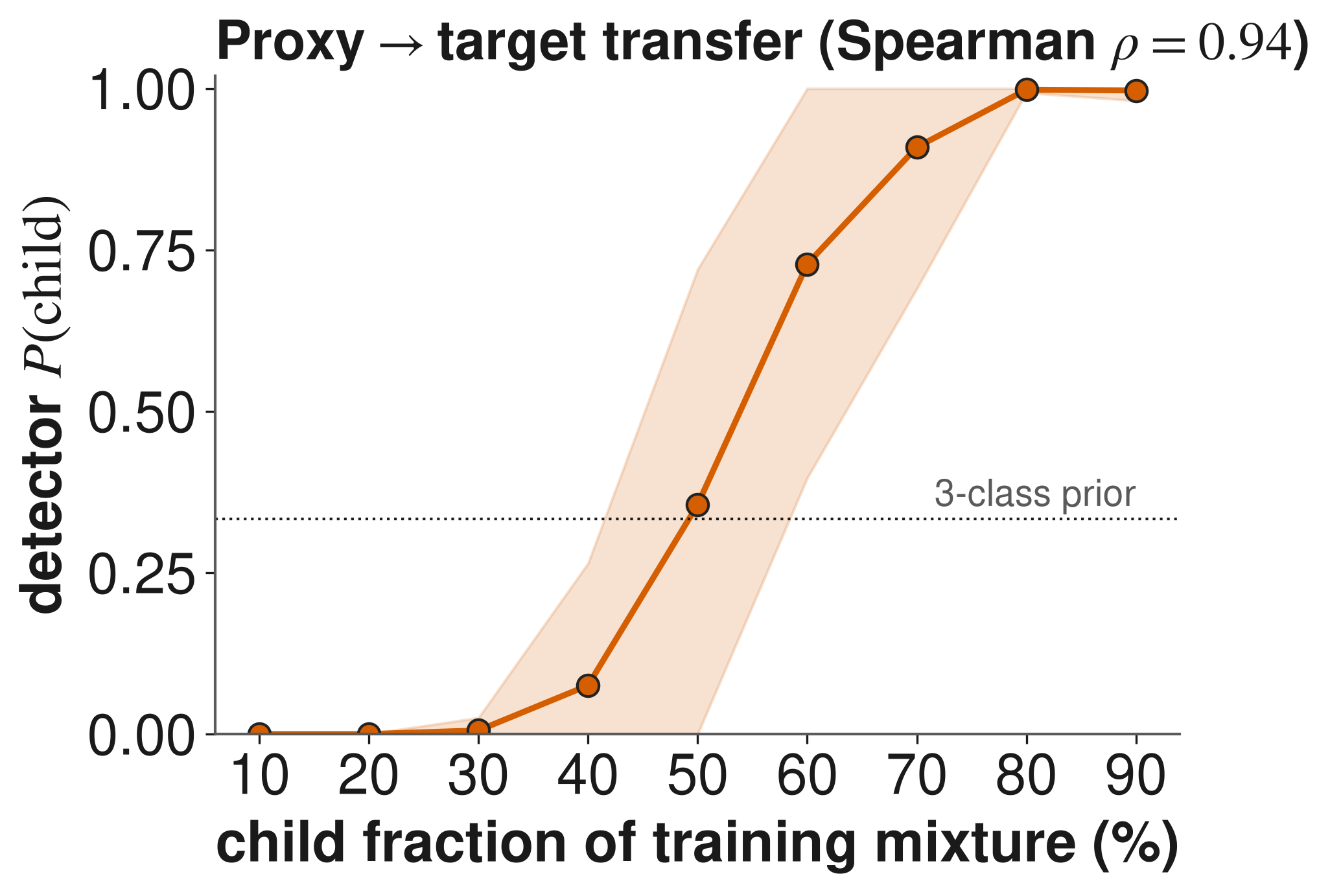}
\caption{Proxy$\to$target transfer. A discrete adult/child/youth detector, applied to graded adult--child mixture adapters it never saw, raises its child probability monotonically with the child fraction of the mixture (Spearman $\rho=0.94$); shaded band is $\pm1$ s.d.\ across the 30 adapters per level.}
\label{fig:transfer}
\end{figure}

\textbf{Robustness to weight-space evasion.} A motivated uploader could perturb an adapter to evade a weight-space screen, so we take target-axis (age) adapters, apply the perturbations such an adversary would use, and re-score them with a clean-trained detector. The $u_1$ signal is essentially unchanged under additive weight noise up to the magnitude of the weights themselves (AUROC $0.99\!\to\!0.98$ at noise scale $\sigma=1$, $0.97$ at $\sigma=2$), and is invariant to global rescaling and an fp16 round-trip (\Cref{tab:evasion-full-main}). 

\begin{table}[h]\centering\small
\begin{tabular}{@{}llc@{}}
\toprule
Perturbation & Strength & age AUROC \\
\midrule
(clean) & --- & 0.992 \\
\midrule
additive noise & $\sigma{=}0.25$ & 0.991 \\
               & $\sigma{=}0.5$  & 0.989 \\
               & $\sigma{=}1.0$  & 0.985 \\
               & $\sigma{=}2.0$  & 0.971 \\
rescale        & $\times0.25,\,\times4$ & 0.992 \\
fp16 round-trip & --- & 0.992 \\
8-bit quantize  & --- & 0.992 \\
4-bit quantize  & --- & 0.990 \\
norm-equalize   & --- & 0.992 \\

\bottomrule
\end{tabular}
\caption{\textbf{Full weight-space evasion sweep: age-detection AUROC of a clean-trained $u_1$ detector applied to perturbed target adapters.} The signal is invariant to rescaling, precision reduction, and norm-equalization, robust to additive noise up to the weight scale.}
\label{tab:evasion-full-main}
\end{table}

\begin{figure}[!h]
\centering
\includegraphics[width=1\columnwidth]{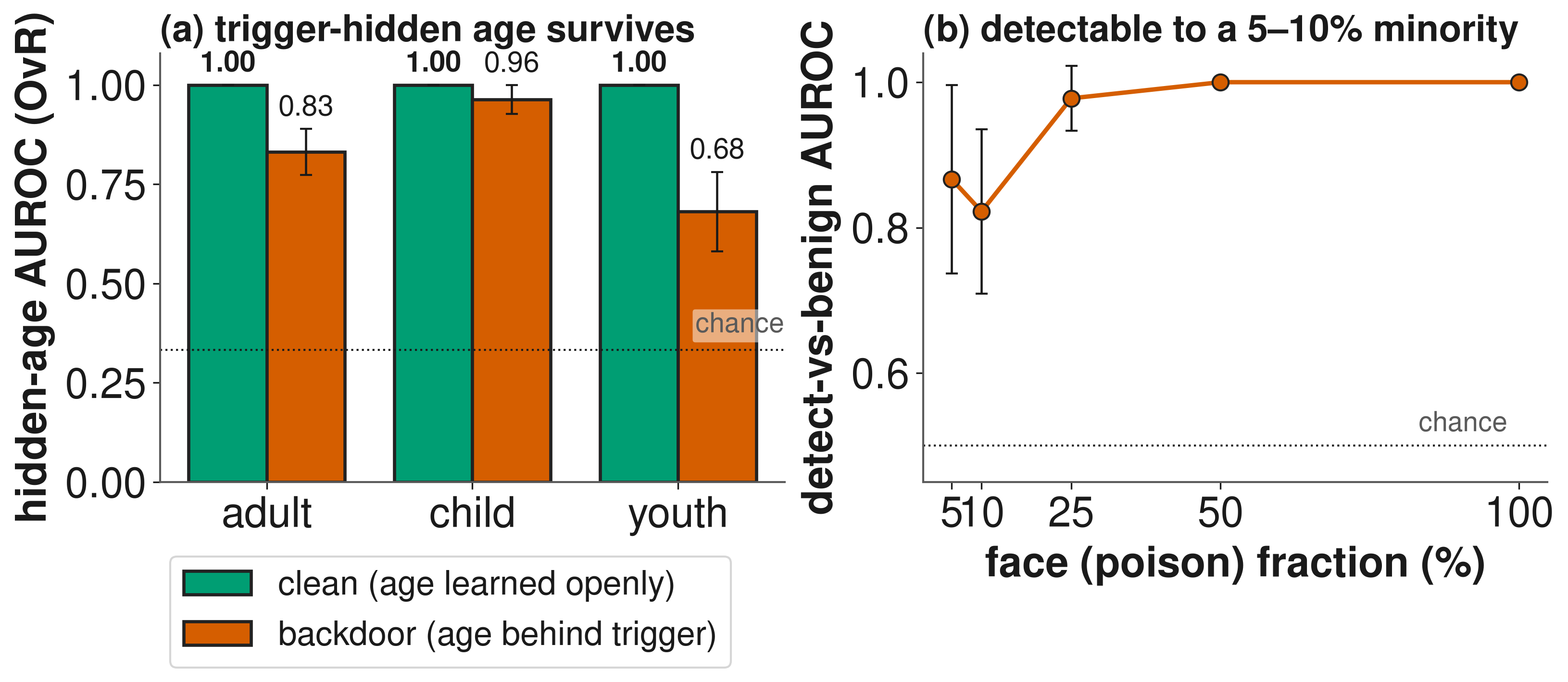}
\caption{\textbf{Recovery from stealth adapters.} \textbf{(a)} Per-age hidden-age AUROC: when age is learned openly (clean), $u_1$ recovers it perfectly; bound behind a trigger token (backdoor) it is still recovered well above the 0.33 chance floor, most strongly for the safety-critical child class. \textbf{(b)} Blended poison-fraction sweep: $u_1$ separates an adapter carrying a face (poison) minority from a pure-benign one down to a 5--10\% mix.}
\label{fig:stealth}
\end{figure}

\textbf{The signal survives a trigger-bound backdoor.} We consider an adapter that is designed to look benign, to produce certain subjects only when provided with a rare trigger, while regularized with innocuous imagery. We construct this case on the age axis, using 60 SD-1.5 adapters in which apparent age is learned only behind a DreamBooth-style trigger token (\texttt{sks person}) and regularized with a benign class \citep{ruiz2023dreambooth}, alongside 60 matched ``clean'' adapters that learn age openly. Within the clean arm, $u_1$ recovers the three-way age axis perfectly (macro one-vs-rest AUROC $1.00$). Within the backdoor arm, where age is never the adapter's nominal task, $u_1$ still recovers it at $0.83$ macro AUROC, far above the $0.33$ chance floor. The degradation is uneven and, for our purposes, favorable: the safety-critical \emph{child} class remains the most recoverable ($0.96$), while the adult/child boundary class (youth) is weakest ($0.68$, \Cref{fig:stealth}a). Separately, the backdoor and clean adapters are themselves perfectly distinguishable in weight space ($1.00$ AUROC). The sensitivity also holds when the poison is a small fraction of the data rather than a trigger: on a sweep that mixes face imagery into an otherwise benign adapter, $u_1$ separates the mixed adapter from a pure-benign one down to a $5$--$10\%$ face minority (AUROC $0.80$--$0.88$), saturating to $1.00$ by a $50\%$ mix (\Cref{fig:stealth}b).

\section{Limitations}\label{sec:limitations}

We test only on benign proxies. Although our detector abstains on broad benign content, it is not a guarantee of generalization to CSAM-generating LoRAs that include children and sexual content. While our method works across U-Net layers, we did not test our method on the text encoder; thus, understanding whether harmful LoRAs which modify the text encoder alone could be identified by our method is an open area for future work. In addition, our detectors are base-model-specific, though the methods transfer to the different models we test. Finally, our model zoos are balanced with one concept per adapter, whereas adapters in the wild may vary more in concept count, captioning, proportions, or include further operations like stacking or repeated fine-tuning.

\section{Ethics and Broader Impact}

We never train, fine-tune, or generate CSAM or any sexual-content-involving-minors model. Every adapter we train is a benign proxy or control. As such, we rely on benign proxies by design. We draw all training data from curated, safe-for-work open datasets, avoiding the nonconsensual collection practices that the community has been urged to abandon \citep{cintaqia2026stop}, and we build the age proxy from synthetic images rather than real people. Because the method operates on weights, it requires no human exposure to abusive imagery during development or deployment screening. This aims to address a wellbeing constraint that limits conventional red-teaming and moderation work \citep{kale2026position}. Given the sensitive subject matter, we nonetheless implemented a set of researcher wellbeing protocols, approved by our institutional ethics processes. Finally, we limit our most actionable findings. We report weaknesses only at the level defenders need and pair them with mitigations. We do not give the specific locations of the flagged in-the-wild LoRAs, and we withhold detailed attack and evasion methods --- including two additional evasion techniques that were moderately effective, whose details we omit to avoid evader uplift. This ensures the work informs defense without serving as a guide to the harmful content.

\section{Conclusion}

A LoRA encodes its training subject in the leading singular directions of its weight matrix, and that direction persists even when the update's magnitude changes. Reading it takes only linear algebra on the weights, so $u_1$ offers a way to screen harmful adapters from proxy supervision alone, with no prompts, generated images, or GPU inference. $u_1$ is data-efficient, abstains on unrelated benign content, and survives additive noise, rescaling, and precision reduction. It stays competitive with, and complementary to, activation probing while needing no inference, and we show it can recover subjects hidden behind trigger-based backdoors. More broadly, it reinforces that harmful adapters could be screened without generating or handling harmful content. Weight-space screening can serve as one layer of a moderation pipeline, alongside metadata review and perceptual hashing, and could flag harmful LoRAs off-platform, where no metadata exists. The results also suggest $u_1$ could support mitigation as well as detection, suppressing a concept in weight space rather than only flagging it, which is an avenue for future work \citep{arditi2024refusal}. We hope trust-and-safety and law-enforcement partners will evaluate and extend it in more realistic settings.

\section*{Acknowledgments}
The authors would like to thank various people for discussions, such as Punya Pandey, Rob Wang, Cameron Holmes, Konstantin Sietzy, Ashia Wilson, Vinith Suriyakumar, and many others whom we've missed due to our own forgetfulness. We would also like to thank the UK AI Security Institute and the Department of Science, Innovation, and Technology more broadly for their support. Our compute-intensive research was only made possible by the generous support of the Bristol Centre for Supercomputing, who provided access to Isambard.

\bibliography{aaai2027}

\appendix
\setcounter{secnumdepth}{2}

\section{Implementation and Experimental Details}
\label{app:implementation}

\paragraph{$u_1$ extraction.} As described in the main text, we extract $u_1$ using a reduced QR factorization, sign-align and concatenate the per-layer leading directions over the cross-attention matrices. We select cross-attention to follow the approach of \citet{duszenko2025towards}. The dimensionality and layers differ, so we build these separately for each base model architecture. Concretely, on SD-1.5 the shared matrices are the 16 value projections of the U-Net cross-attention blocks (6 encoder, 1 mid, 9 decoder), giving a 12,480-dimensional feature; on SDXL, 70 cross-attention matrices give 83,200 dimensions; and on FLUX.1-schnell, whose DiT has no U-Net, we read the 19 double-stream attention blocks for a 58,368-dimensional feature.

\paragraph{Classifiers and protocol.} Unless noted we use $\ell_2$-regularized logistic regression ($C{=}1$, standardized features) and report macro one-vs-rest AUROC under 5-fold stratified cross-validation; error bars are the standard deviation across folds, and the headline confidence intervals (\Cref{tab:perclass-full}) are the $95\%$ interval over eight CV reshuffles. \Cref{tab:clfsweep} shows the result is stable across classifier families.

\paragraph{Training zoos.} Each adapter is trained on a single concept with a distinct image subset drawn from a larger pool, under a per-adapter randomized recipe: rank $\in\{8,16,32\}$, learning rate $\in\{5{\times}10^{-5},10^{-4},2{\times}10^{-4}\}$, optimizer $\in\{\text{Adam},\text{AdamW}\}$, $100$ epochs, $512^2$ resolution, with light augmentation (random resized crop, horizontal flip, color jitter) to prevent replicate fingerprints. This randomization is exactly what the disentanglement control (\Cref{fig:disentangle}) exploits: the recipe is not recoverable from $u_1$. \Cref{tab:zoo-inventory} inventories every adapter zoo trained for this work ($\approx$4{,}730 single-concept adapters in total).

\begin{figure}[h]
\centering
\includegraphics[width=\columnwidth]{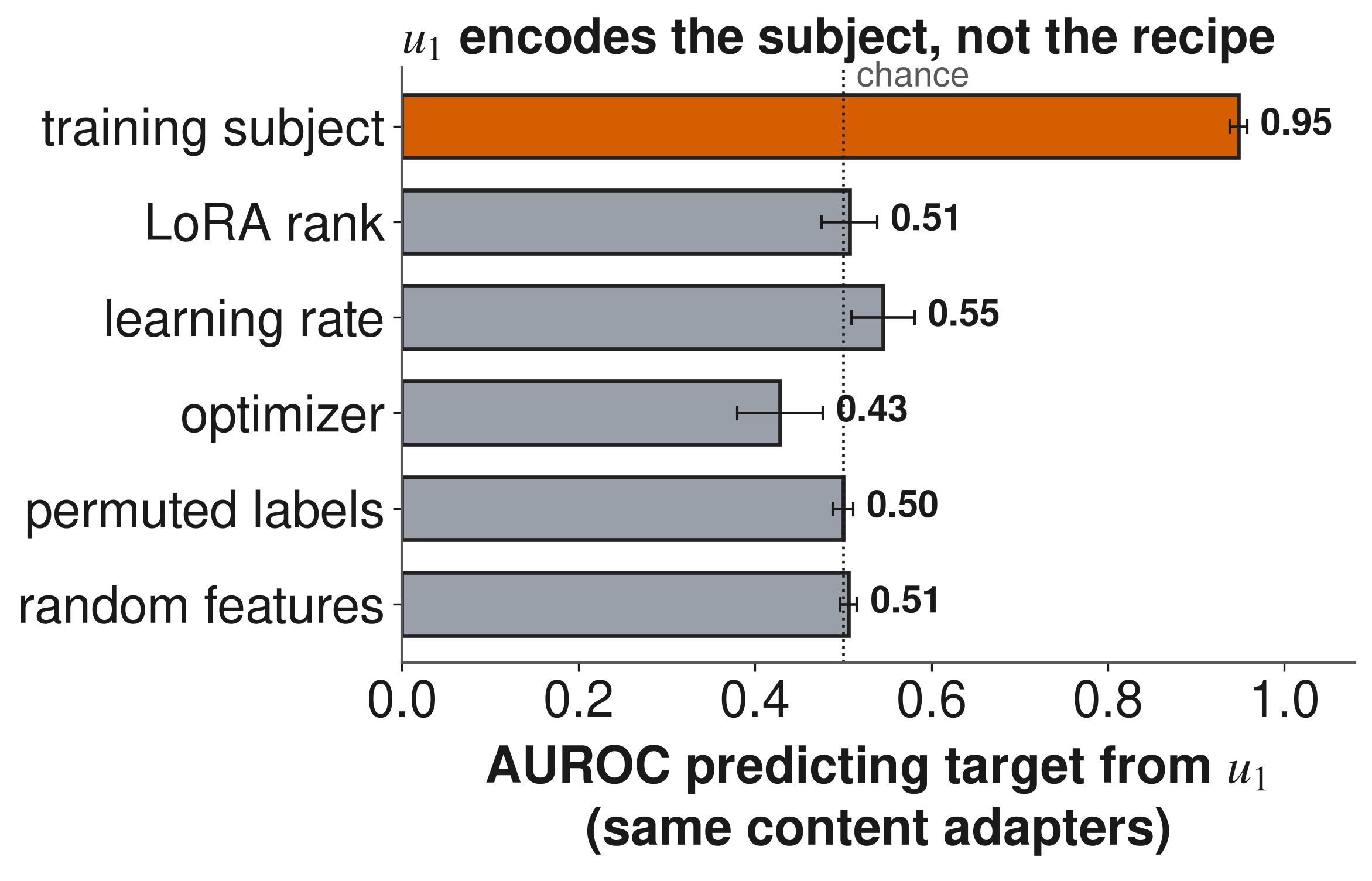}
\caption{\textbf{$u_1$ encodes the training subject, not the recipe.} AUROC of a probe predicting each target from $u_1$ on the \emph{same} content adapters: the subject is recovered at 0.95, while the per-adapter randomized recipe (rank, learning rate, optimizer) and the label-permutation / random-feature negative controls all sit at chance.}
\label{fig:disentangle}
\end{figure}

\begin{table}[t]\centering\small
\begin{tabular}{@{}llr@{}}
\toprule
Zoo / role & Base & $N$ \\
\midrule
\multicolumn{3}{@{}l}{\textit{SD-1.5 subject zoos (source of the 7-category benchmark)}}\\
content (birds/scenes/objects/textures/land) & SD-1.5 & 627 \\
age (adult/child/youth)                      & SD-1.5 & 300 \\
iNaturalist genus                            & SD-1.5 & 300 \\
\midrule
\multicolumn{3}{@{}l}{\textit{Cross-base (method transfer)}}\\
content           & SDXL & 627 \\
age               & SDXL & 293 \\
iNaturalist genus & SDXL & 300 \\
content           & FLUX.1-schnell & 627 \\
age               & FLUX.1-schnell & 150 \\
iNaturalist genus & FLUX.1-schnell & 300 \\
\midrule
\multicolumn{3}{@{}l}{\textit{Controls}}\\
random-caption (genus)            & SD-1.5 & 150 \\
random-caption (age, matched)     & SD-1.5 & 120 \\
benign person    & SD-1.5 & 100 \\
attn-vs-feed-forward ablation     & SD-1.5 & 120 \\
\midrule
\multicolumn{3}{@{}l}{\textit{Capstones (stealth / transfer)}}\\
age behind trigger (clean+backdoor) & SD-1.5 & 120 \\
species behind trigger              & SD-1.5 & 240 \\
blended poison-fraction sweep       & SD-1.5 & 90 \\
age-mix transfer (adult--child)     & SD-1.5 & 270 \\
\bottomrule
\end{tabular}
\caption{Adapters used per experiment. All are benign, single-concept, trained from curated open datasets.}
\label{tab:zoo-inventory}
\end{table}

\paragraph{Gaussian-probe baseline.} We reproduce activation probing \citep{suriyakumar2026evaluation} by loading each adapter into the base model, running DDIM denoising from $1{,}024$ Gaussian latents under null text conditioning ($30$ steps), mean-pooling the intermediate activation, and training the same classifiers. Unlike $u_1$ this requires GPU inference and a full denoising trajectory per adapter.

\paragraph{Compute.} Adapters were trained as throttled Slurm array jobs on GH200 GPUs; all $u_1$ analysis is CPU-only linear algebra on the adapter weights and requires no model execution.

\section{Theoretical Details for the $u_1$ Representation}
\label{app:u1-theory}

This appendix proves the claims used to motivate the $u_1$ representation in the main text (Weight-space Detection). Throughout, singular vectors should be understood as unsigned directions $[u]=\{u,-u\}$ unless a signed representative is explicitly chosen; we measure distance between directions by $\sin\angle(u,\tilde u)=\sqrt{1-(u^\top\tilde u)^2}=\norm{uu^\top-\tilde u\tilde u^\top}_2$. The projective convention is what the per-fold sign alignment in our implementation realizes numerically.

\begin{proposition}[Scale invariance]
\label{prop:scale-invariance}
Let $M\in\R^{m\times n}$ have a simple largest singular value, and let $c\neq 0$. Then $[u_1(cM)]=[u_1(M)]$ and $u_1(cM)u_1(cM)^\top=u_1(M)u_1(M)^\top$.
\end{proposition}
\begin{proof}
Multiplication by $c$ sends each singular value to $|c|\sigma_i$ and preserves the left singular subspaces. For $c<0$ the sign is absorbed into either singular vector; since $(u_i,v_i)$ and $(-u_i,-v_i)$ encode the same rank-one term, only the projector $u_iu_i^\top$ is intrinsic.
\end{proof}

\begin{proposition}[Rank-one subject model]
\label{prop:rank-one-subject}
Let $M=\gamma w a^\top$ with $w\neq0$, $a\neq0$, $\gamma\neq0$. Then $M$ has rank one, its only nonzero singular value is $|\gamma|\norm{w}_2\norm{a}_2$, and $[u_1(M)]=[w]$.
\end{proposition}
\begin{proof}
The identity $\gamma w a^\top=|\gamma|\norm{w}_2\norm{a}_2\bigl(\sgn(\gamma)\tfrac{w}{\norm{w}_2}\bigr)\bigl(\tfrac{a}{\norm{a}_2}\bigr)^\top$ exhibits an SVD of a nonzero rank-one matrix, whose nonzero left singular space is $\operatorname{span}\{w\}$.
\end{proof}

\begin{theorem}[One-dimensional Wedin bound]
\label{thm:one-dim-wedin}
Let $M\in\R^{m\times n}$ have singular values $\sigma_1>\sigma_2\geq\cdots$, top singular vectors $u,v$, and gap $g=\sigma_1-\sigma_2$. Let $\widetilde M=M+E$, $\varepsilon=\norm{E}_2$, and let $\tilde u,\tilde v$ be a leading singular pair of $\widetilde M$ with singular value $\tilde\sigma_1$. If $\tilde\sigma_1>\sigma_2$ then
\begin{equation}
    \max\{\sin\angle(u,\tilde u),\sin\angle(v,\tilde v)\}\leq\frac{\varepsilon}{\tilde\sigma_1-\sigma_2}.
    \label{eq:posteriori-wedin}
\end{equation}
Consequently, by Weyl's inequality, if $\varepsilon<g$ then the bound is at most $\varepsilon/(g-\varepsilon)$; and if $\varepsilon<g/2$ the perturbed largest singular value is simple and the bound is at most $2\varepsilon/g$.
\end{theorem}
\begin{proof}
Let $s_u=\norm{(I-uu^\top)\tilde u}_2$ and $s_v=\norm{(I-vv^\top)\tilde v}_2$. The leading singular equations $\widetilde M\tilde v=\tilde\sigma_1\tilde u$ and $\widetilde M^\top\tilde u=\tilde\sigma_1\tilde v$, projected orthogonally to $u$ and to $v$ respectively, give $\tilde\sigma_1 s_u\leq\sigma_2 s_v+\varepsilon$ and $\tilde\sigma_1 s_v\leq\sigma_2 s_u+\varepsilon$ (using $Mv=\sigma_1u$, $M^\top u=\sigma_1 v$, and that $M$ has operator norm $\sigma_2$ on $v^\perp$). With $s=\max\{s_u,s_v\}$ these give $(\tilde\sigma_1-\sigma_2)s\leq\varepsilon$, proving \eqref{eq:posteriori-wedin}. Weyl's inequality gives $\tilde\sigma_1\geq\sigma_1-\varepsilon$, hence $\tilde\sigma_1-\sigma_2\geq g-\varepsilon$; if $\varepsilon<g/2$ a second application gives $\tilde\sigma_1-\tilde\sigma_2\geq g-2\varepsilon>0$, so the perturbed top singular value is simple and $\varepsilon/(g-\varepsilon)\leq2\varepsilon/g$.
\end{proof}

\begin{proposition}[Shared output-gradient model]
\label{prop:shared-gradient}
For a linear projection $y=Wx$, suppose the subject-specific output-side gradients along the trajectory are $g_t=\lambda_t w+\xi_t$. A full-matrix first-order update with step $\eta$ is $\Delta W=-\eta w a^\top+E$, with $a=\sum_t\lambda_t x_t$ and $E=-\eta\sum_t\xi_t x_t^\top$. If $a\neq0$ and $\norm{E}_2<\eta\norm{w}_2\norm{a}_2$, then $\sin\angle(u_1(\Delta W),w)\leq \norm{E}_2/(\eta\norm{w}_2\norm{a}_2-\norm{E}_2)$.
\end{proposition}
\begin{proof}
Gradient descent accumulates $-\eta\sum_t g_t x_t^\top$; substituting $g_t=\lambda_t w+\xi_t$ gives the decomposition. The rank-one part has singular value $\eta\norm{w}_2\norm{a}_2$ and second singular value $0$, so \Cref{thm:one-dim-wedin} applies with gap $g=\eta\norm{w}_2\norm{a}_2$ and perturbation $E$.
\end{proof}

\begin{lemma}[Early-time LoRA bottleneck]
\label{lem:lora-bottleneck}
Let $\Delta W=BA$ with $B_0=0$ and $A_0$ fixed at initialization. If the full weight-gradient signal is $G=wa^\top$, the first gradient step on $B$ gives $B_1A_0=-\eta\, w\,(A_0^\top A_0 a)^\top$, so the bottleneck projects the cue side but preserves the output-side left direction whenever $A_0^\top A_0 a\neq0$.
\end{lemma}
\begin{proof}
The gradient with respect to $B$ is $GA_0^\top$, so $B_1=-\eta w a^\top A_0^\top$ and $B_1A_0=-\eta w a^\top A_0^\top A_0$.
\end{proof}

\begin{corollary}[Stable rank and approximate rank one]
\label{cor:stable-rank}
If $M=\tau uv^\top+R$ with $\tau>0$, $\norm{u}_2=\norm{v}_2=1$, and $\norm{R}_2<\tau$, then $\sin\angle(u_1(M),u)\leq\norm{R}_2/(\tau-\norm{R}_2)$. Moreover $\sr(M)=\norm{M}_F^2/\norm{M}_2^2\leq1+\rho^2$ implies $\sigma_2(M)/\sigma_1(M)\leq\rho$ and $\sigma_1-\sigma_2\geq(1-\rho)\sigma_1$, so stable rank near one is an observable certificate of a dominant direction (sufficient, not necessary).
\end{corollary}
\begin{proof}
Apply \Cref{thm:one-dim-wedin} to $\tau uv^\top$ (gap $\tau$) with perturbation $R$. The stable-rank claim is immediate from $\sr(M)=1+\sum_{i\geq2}(\sigma_i/\sigma_1)^2$.
\end{proof}

\begin{corollary}[Merge retention]
\label{cor:merge-retention}
Let $M_\alpha=(1-\alpha)H+\alpha B$, with $H$ having top left direction $u_h$ and gap $g_h>0$. If $\alpha\norm{B}_2<(1-\alpha)g_h$ then $\sin\angle(u_1(M_\alpha),u_h)\leq \alpha\norm{B}_2/((1-\alpha)g_h-\alpha\norm{B}_2)$.
\end{corollary}
\begin{proof}
Apply \Cref{thm:one-dim-wedin} to the base matrix $(1-\alpha)H$ (leading direction $u_h$, gap $(1-\alpha)g_h$) with perturbation $\alpha B$.
\end{proof}

\begin{proposition}[Explicit orthogonal merge threshold]
\label{prop:orthogonal-merge}
Let $H=\tau_h u_h v_h^\top$ and $B=\tau_b u_b v_b^\top$ with $u_h\perp u_b$, $v_h\perp v_b$, $\tau_h,\tau_b>0$. The leading left direction of $M_\alpha=(1-\alpha)H+\alpha B$ is $[u_h]$ when $\alpha<\alpha_\star$ and $[u_b]$ when $\alpha>\alpha_\star$, where $\alpha_\star=\tau_h/(\tau_h+\tau_b)$; at $\alpha_\star$ the top singular value is repeated and the leading direction is not unique.
\end{proposition}
\begin{proof}
The two rank-one terms have mutually orthogonal left and right singular vectors, so $M_\alpha$ has singular values $(1-\alpha)\tau_h$ and $\alpha\tau_b$; comparing them gives the threshold.
\end{proof}

\begin{proposition}[Top-$k$ merge condition, orthogonal model]
\label{prop:topk-merge}
Let $H=\tau_h u_h v_h^\top$ and $B=\sum_{i=1}^q\beta_i b_i c_i^\top$ with $\beta_1\geq\cdots\geq\beta_q>0$ and all displayed directions mutually orthogonal. In $M_\alpha=(1-\alpha)H+\alpha B$ the target direction $u_h$ lies in the top-$k$ left singular subspace iff $(1-\alpha)\tau_h>\alpha\beta_k$ (with the usual non-uniqueness at equality).
\end{proposition}
\begin{proof}
The singular values of $M_\alpha$ are $(1-\alpha)\tau_h$ together with $\alpha\beta_1,\dots,\alpha\beta_q$; $u_h$ is among the top $k$ exactly when fewer than $k$ donor values strictly exceed $(1-\alpha)\tau_h$.
\end{proof}

For a general top-$k$ subspace, Wedin's theorem \cite{Wedin1972,StewartSun1990} gives the analogous stability statement $\norm{\sin\Theta(U_k,\widetilde U_k)}_2\leq\norm{E}_2/(\sigma_k-\sigma_{k+1}-\norm{E}_2)$ whenever the denominator is positive. Multi-component features therefore mitigate only the failure mode in which the target is displaced from rank one but remains within a well-separated leading subspace; if the top-$k$ gap is small, the subspace itself is unstable. These are statements about identifiability of directions in adapter weight space, not certificates of what the composed model can generate.

\section{Additional Results}
\label{app:results}

\paragraph{Per-category results.} \Cref{tab:perclass-full} gives the
per-category one-vs-rest AUROC with multi-seed confidence intervals and the negative
controls behind the headline macro result, \Cref{tab:clfsweep} the classifier-capacity
sweep, and
\Cref{fig:confusion} the class-balanced confusion structure, where errors concentrate
among the visually similar natural-image categories while age is the most separable.

\begin{table}[t]\centering\small
\begin{tabular}{@{}lcc@{}}
\toprule
Category & OvR AUROC & 95\% CI \\
\midrule
age (adult/child/youth) & 0.998 & $\pm$0.000 \\
land cover              & 1.000 & $\pm$0.000 \\
objects                 & 0.989 & $\pm$0.000 \\
genus                   & 0.984 & $\pm$0.001 \\
scenes                  & 0.963 & $\pm$0.001 \\
textures                & 0.951 & $\pm$0.002 \\
birds                   & 0.945 & $\pm$0.002 \\
\midrule
\textbf{macro}          & \textbf{0.976} & $\pm$0.001 \\
label-permuted (control) & 0.500 & $\pm$0.012 \\
random-feature (control) & 0.506 & $\pm$0.010 \\
\bottomrule
\end{tabular}
\caption{Per-category in-distribution one-vs-rest AUROC (logistic regression on $u_1$, SD-1.5 benchmark), mean and $95\%$ CI over eight CV reshuffles, with
negative controls at chance.}
\label{tab:perclass-full}
\end{table}

\begin{table}[t]\centering\small
\begin{tabular}{@{}lc@{}}
\toprule
Classifier & macro OvR AUROC \\
\midrule
logistic regression & $0.976\pm0.005$ \\
MLP (256 hidden)     & $0.968\pm0.006$ \\
random forest        & $0.940\pm0.010$ \\
\bottomrule
\end{tabular}
\caption{Classifier-capacity sweep on the seven-category benchmark (mean$\pm$std over CV
folds). The result tracks the $u_1$ representation, not classifier capacity---the random
forest lags only because tree ensembles handle the dense high-dimensional feature less
gracefully.}
\label{tab:clfsweep}
\end{table}

\begin{figure}[t]
\centering
\includegraphics[width=0.92\columnwidth]{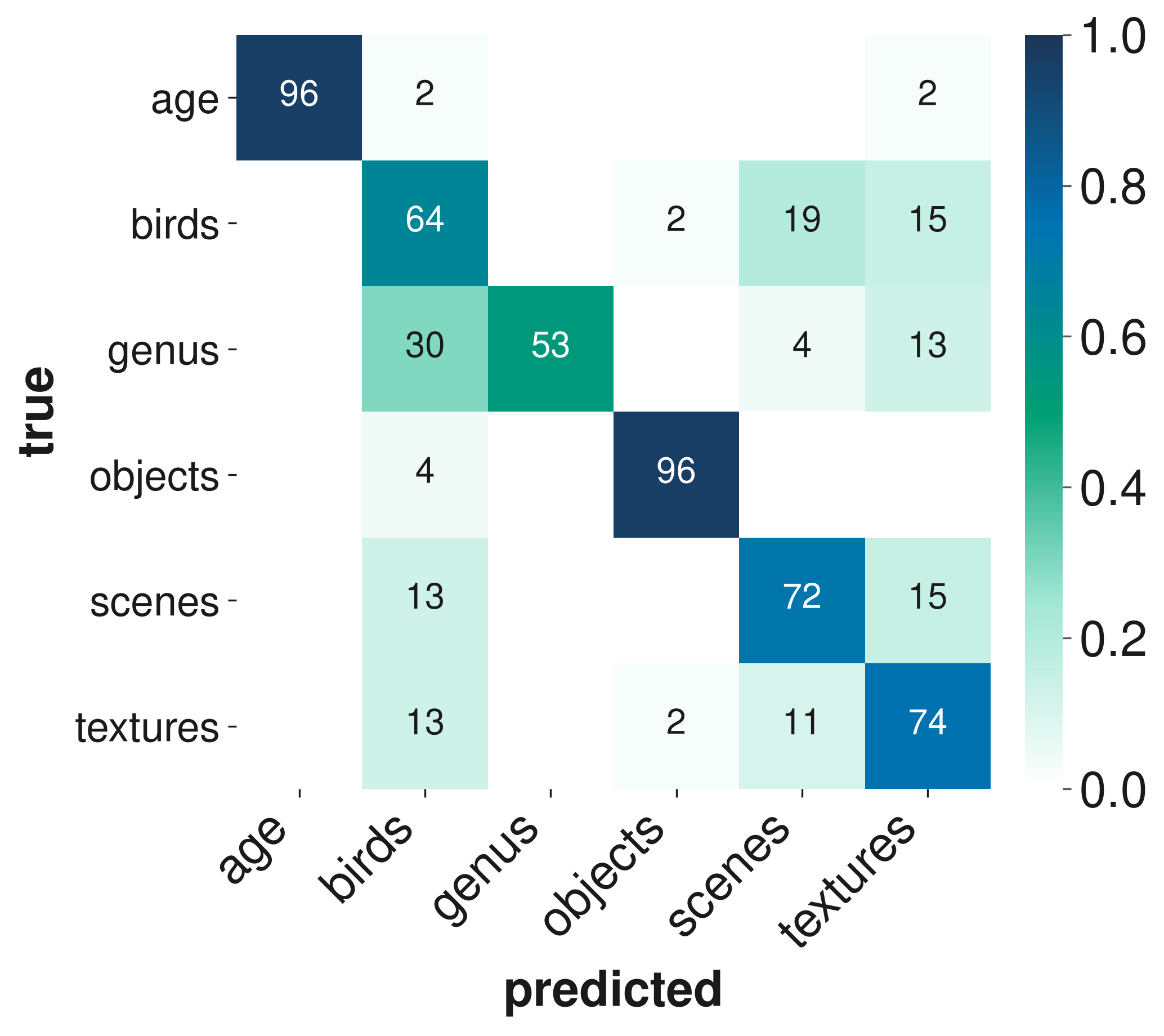}
\caption{Class-balanced confusion (row-normalized \%, logistic regression on
$u_1$, 47 adapters per category). Errors track visual domain---the detailed
natural-image categories mutually confuse---while age, the safety-relevant category, is
the most separable.}
\label{fig:confusion}
\end{figure}

\paragraph{Classification by base model.} \Cref{tab:crossbase} reports $u_1$
across SD-1.5, SDXL, and FLUX.1-schnell.

\begin{table}[t]\centering\small
\begin{tabular}{@{}lccc@{}}
\toprule
Base & 7-way ($N$) & age & genus \\
\midrule
SD-1.5 & 0.976 (957) & 0.993 & 0.987 \\
SDXL   & 0.997 (948) & 0.987 & 0.940 \\
FLUX.1-schnell   & 1.000 (807) & 0.978 & 0.998 \\
\bottomrule
\end{tabular}
\caption{$u_1$ subject recovery by base model (logistic regression, 5-fold
CV, macro one-vs-rest AUROC): the seven-category benchmark and the two single-axis (age,
genus) zoos.}
\label{tab:crossbase}
\end{table}

\textbf{The subject signal is distributed throughout the network.}
On SD-1.5, the subject signal is distributed across the cross-attention layers: strong in the
encoder (0.963) and decoder (0.959), weakest at the mid-block (0.783), and best when all
layers are combined (0.976). It is also data-efficient, reaching $\approx$0.97 macro AUROC
by 80 adapters per category. For other models, the subjects are also recoverable throughout the network.

\paragraph{Where the subject signal lies.} \Cref{fig:localization} plots the
per-region localization on SD-1.5 where the mid-block is weakest and aggregating all layers beats any single region, and \Cref{fig:learning} the data-efficiency curve. On FLUX.1-schnell's DiT the signal is instead uniform across blocks ($0.99$--$1.00$ seven-way, $0.94$--$1.00$
for age).

\begin{figure}[t]
\centering
\includegraphics[width=0.92\columnwidth]{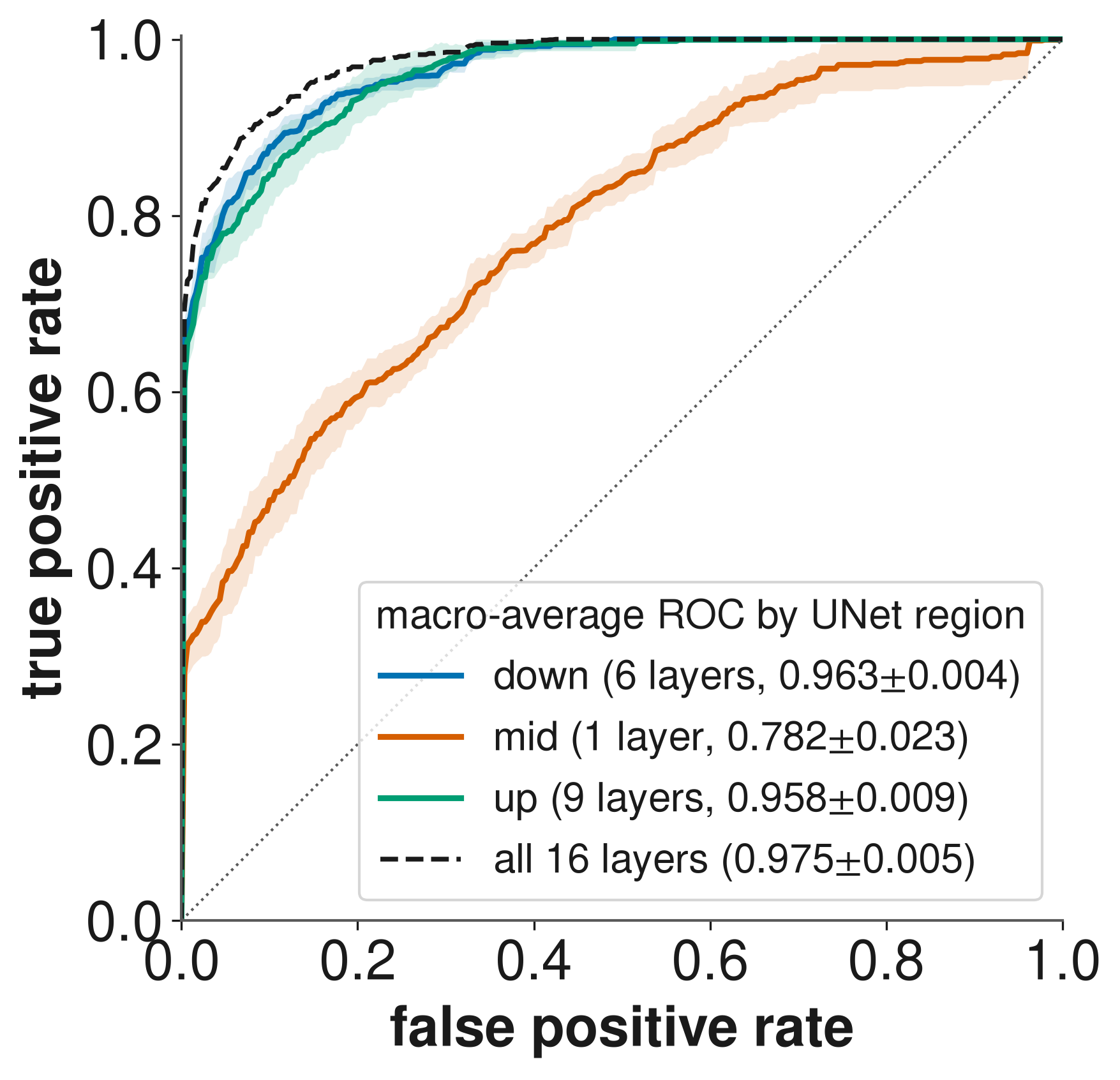}
\caption{Per-U-Net-region macro-average ROC from per-layer $u_1$ for SD-1.5.}
\label{fig:localization}
\end{figure}

\begin{figure}[t]
\centering
\includegraphics[width=0.92\columnwidth]{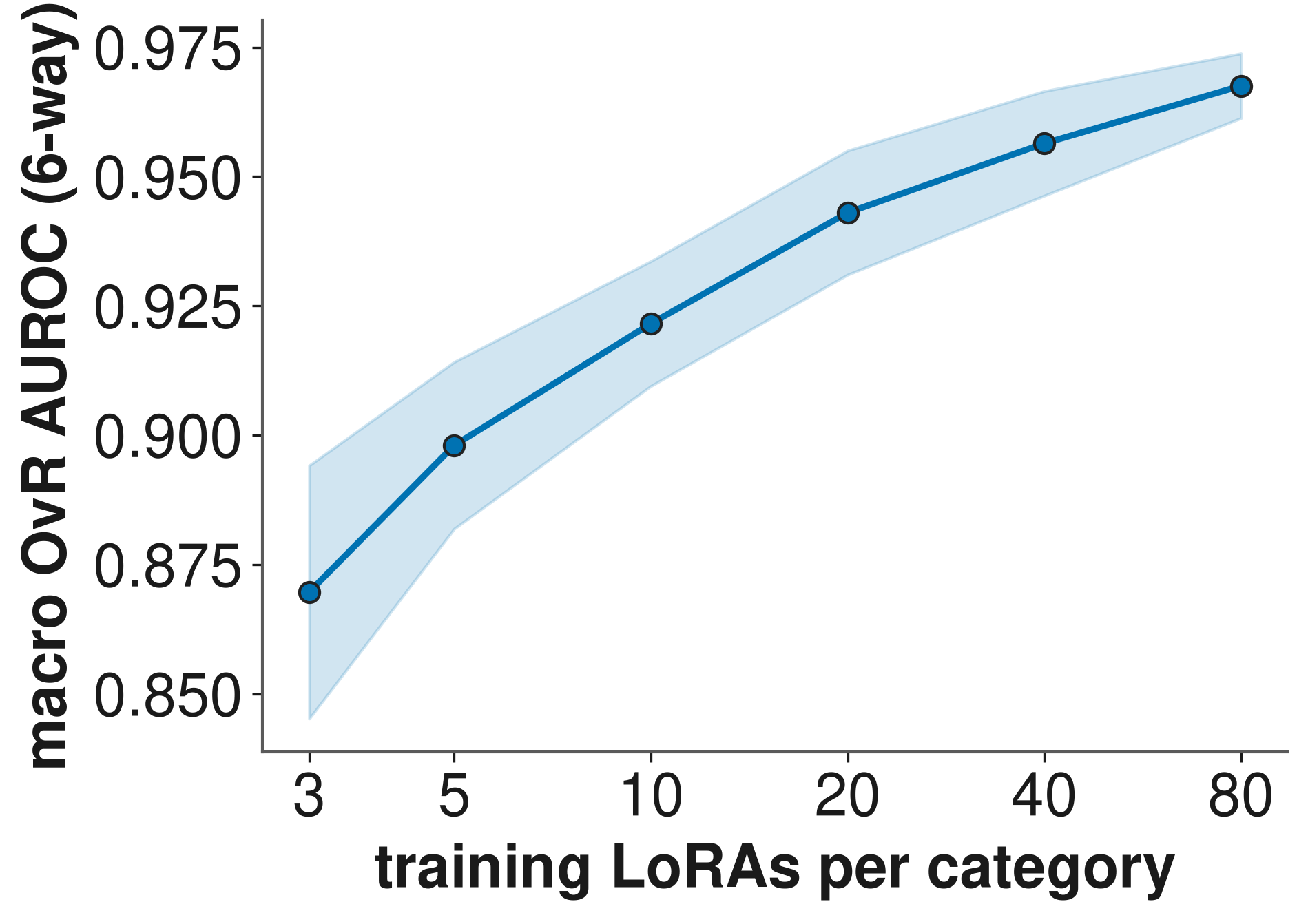}
\caption{Data efficiency: macro one-vs-rest AUROC (the six categories with
$\geq$100 adapters for SD-1.5, fixed held-out test, 12 repeats) against training adapters per
category. The detector reaches $\approx$0.87 from only 3 adapters per category and
saturates near 0.97 by 80.}
\label{fig:learning}
\end{figure}

\paragraph{Attention family and module type.} \Cref{tab:module-ablation}
reports $u_1$ read from a single attention family or module type. On SD-1.5,
cross-attention recovers subject at $0.979$ macro one-vs-rest AUROC and self-attention at
$0.937$ (both together $0.982$, seven-way, $957$ adapters). On SDXL the two families are
near-saturated and comparable (cross $0.997$, self $0.998$; $n{=}948$).

\begin{table}[t]\centering\small
\begin{tabular}{@{}lr@{}}
\toprule
Restriction & macro AUROC \\
\midrule
\emph{Attention family, SD-1.5} (7-way, $n{=}957$) & \\
\quad self-attention                    & 0.937 \\
\quad \textbf{cross-attention (default)} & \textbf{0.979} \\
\quad both                              & 0.982 \\
\midrule
\emph{Attention family, SDXL} (7-way, $n{=}948$) & \\
\quad self-attention                    & 0.998 \\
\quad cross-attention                   & 0.997 \\
\quad both                              & 0.999 \\
\midrule
\emph{Module type, SD-1.5} (4-way content, $n{=}120$) & \\
\quad cross-attention            & 0.998 \\
\quad feed-forward / projection  & 0.998 \\
\quad both                       & 1.000 \\
\bottomrule
\end{tabular}
\caption{$u_1$ read from a single attention family or module type (logistic
regression, 5-fold CV, macro one-vs-rest AUROC)}
\label{tab:module-ablation}
\end{table}

\paragraph{Left vs.\ right singular vector.}\label{app:rightvec}

We compare $u_1$ (the direction the adapter injects into the denoiser) with $v_1$ (the cue side that receives the update) and a concatenation of $u_1$ and $v_1$ on the $957$-adapter SD-1.5 benchmark under the same protocol as
the weight-feature ablation in the main text. $v_1$ alone is less effective in classifying ($0.822$ macro AUROC, $0.38$ balanced accuracy) while the concatenation does not improve on $u_1$ alone ($0.977$ vs $0.979$ AUROC). Therefore, the discriminatory signal lies in the left singular direction.

\paragraph{Representation baselines.} \Cref{fig:baselines-perclass} breaks the baseline
comparison down by category for SD-1.5.

\begin{figure}[t]
\centering
\includegraphics[width=\columnwidth]{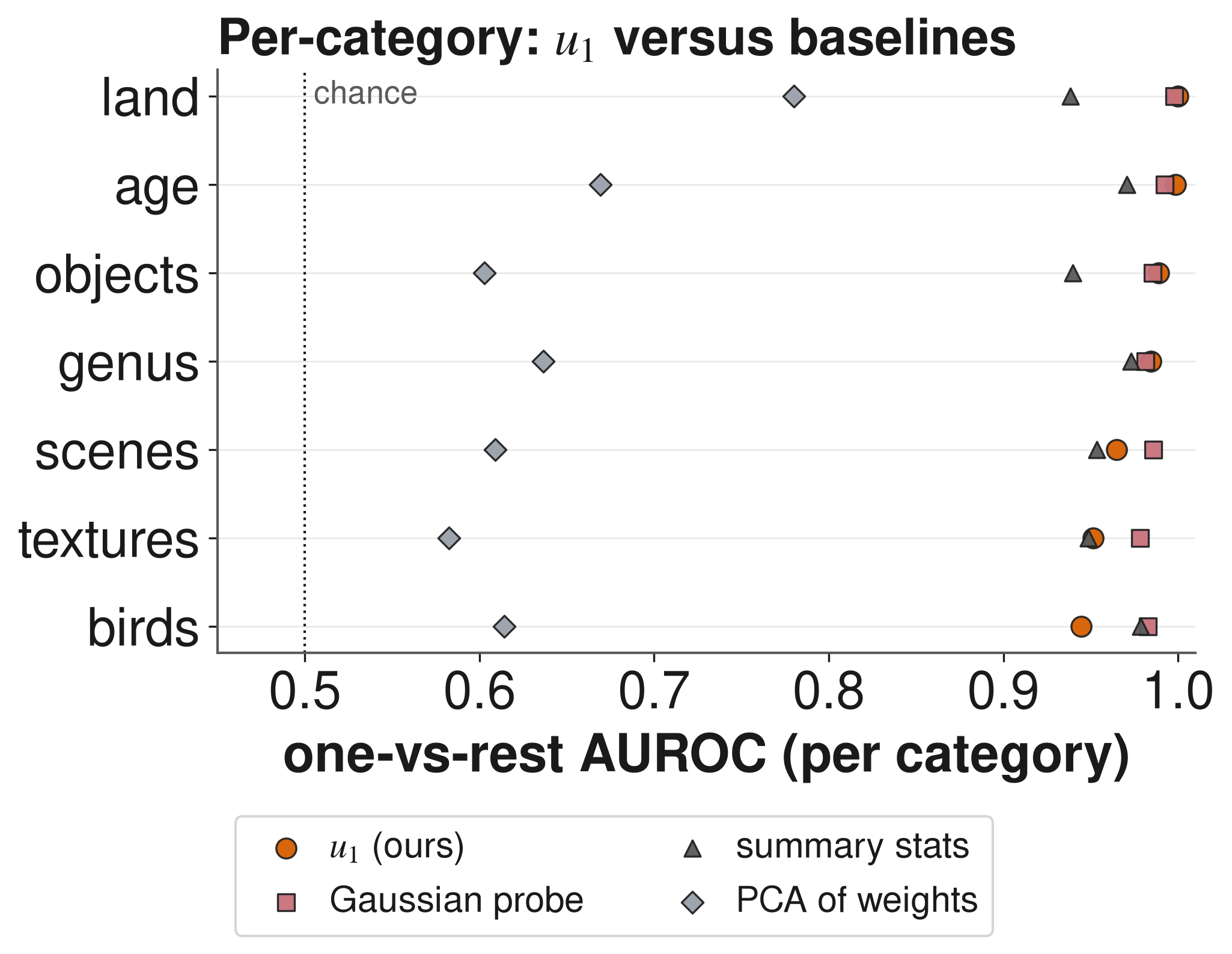}
\caption{Per-category one-vs-rest AUROC for each method for SD-1.5}
\label{fig:baselines-perclass}
\end{figure}

\paragraph{Calibration and geometry.} \Cref{fig:ood} shows the
out-of-distribution calibration behind the main-text abstention result.
\Cref{fig:umap} shows the UMAP projection.

\begin{figure}[t]
\centering
\includegraphics[width=0.92\columnwidth]{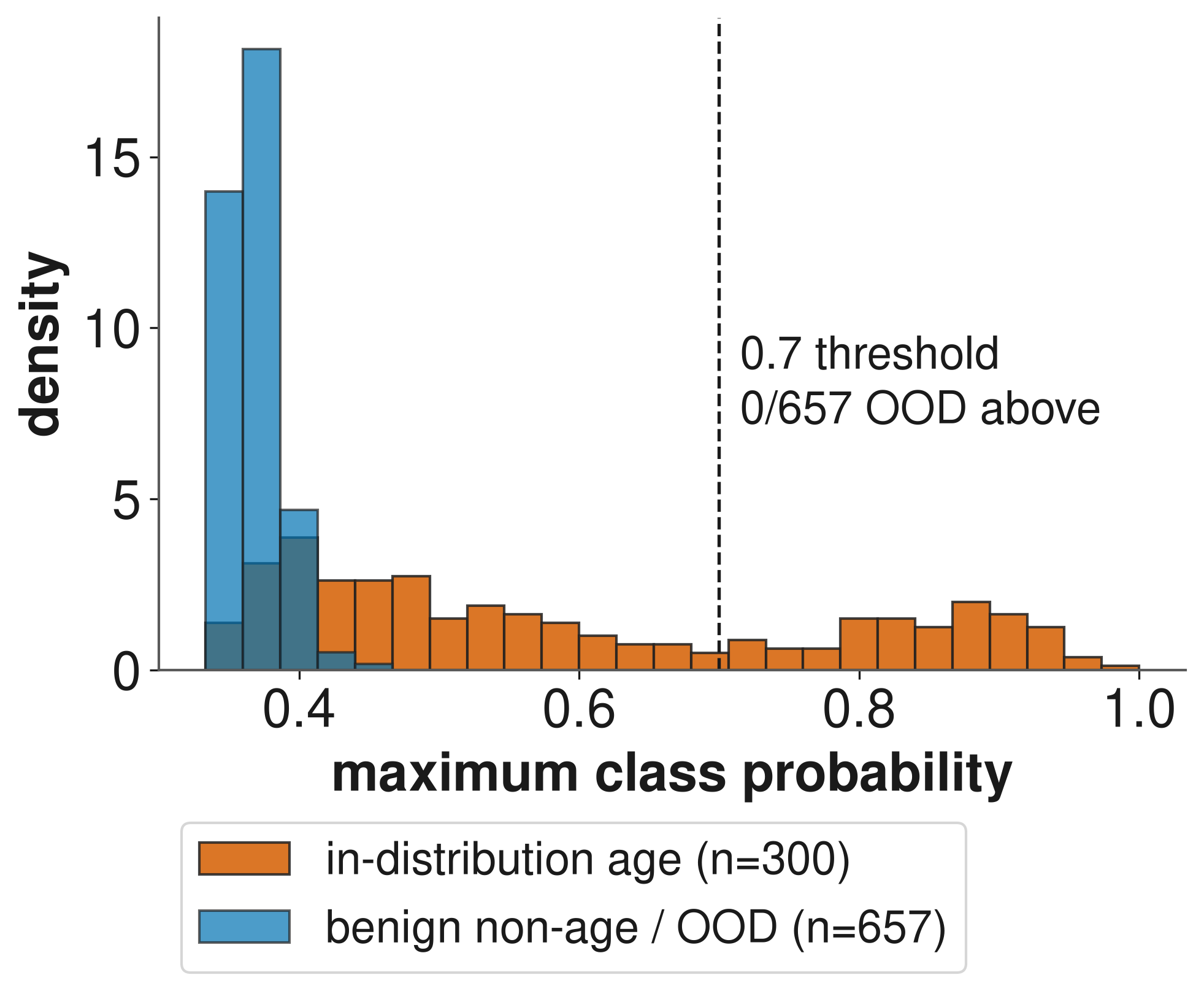}
\caption{Out-of-distribution calibration. An age-proxy detector is confident
in-distribution but abstains on benign non-age content.}
\label{fig:ood}
\end{figure}

\begin{figure}[t]
\centering
\includegraphics[width=\columnwidth]{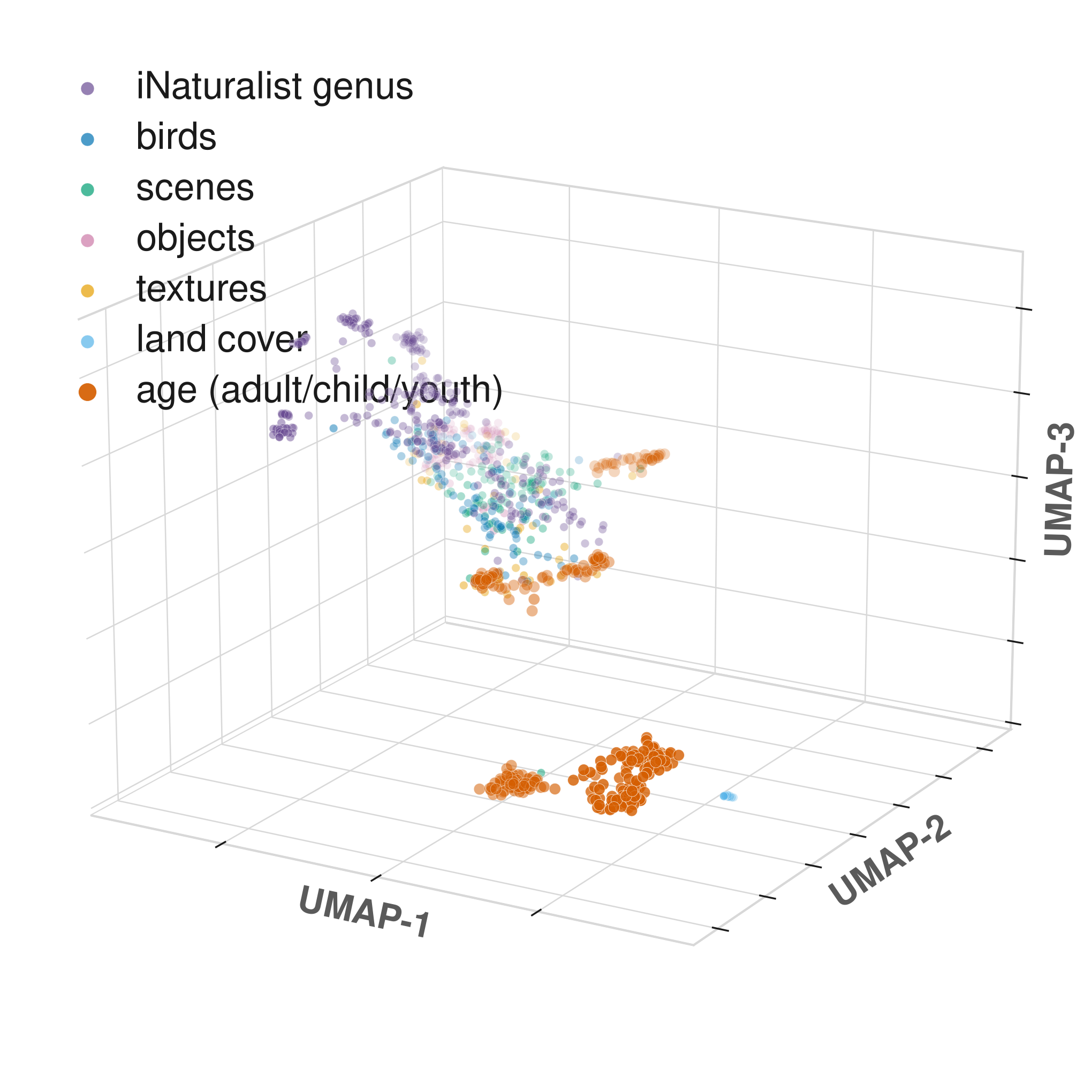}
\caption{3D UMAP of $u_1$ across the SD-1.5 zoos.}
\label{fig:umap}
\end{figure}

\paragraph{Controlling for dataset vs. source.} A natural worry is that we are primarily detecting which dataset a LoRA came from rather than what it depicts, because in the main 7-way benchmark each subject category happens to come from its own dataset, so ``subject'' and ``source'' are confounded. To separate them, we use birds, which are present in two independent sources---CUB-200 and the \emph{melanerpes} (woodpecker) genus in iNaturalist---and train a bird-vs-everything-else detector on birds from only one source, then test whether it recognizes birds from the other, never-seen source (the negatives include the iNaturalist butterfly genera, so the detector cannot succeed by merely spotting ``iNaturalist-ness''). The detector still recovers the unseen source's birds well above chance---$0.77$ AUROC training on CUB and testing on iNaturalist, $0.85$ the reverse---versus $0.957$ when the two bird sources are pooled (\Cref{fig:birds}). So $u_1$ encodes the bird \emph{subject} rather than dataset provenance.

\begin{figure}[t]
    \centering
    \includegraphics[width=0.9\columnwidth]{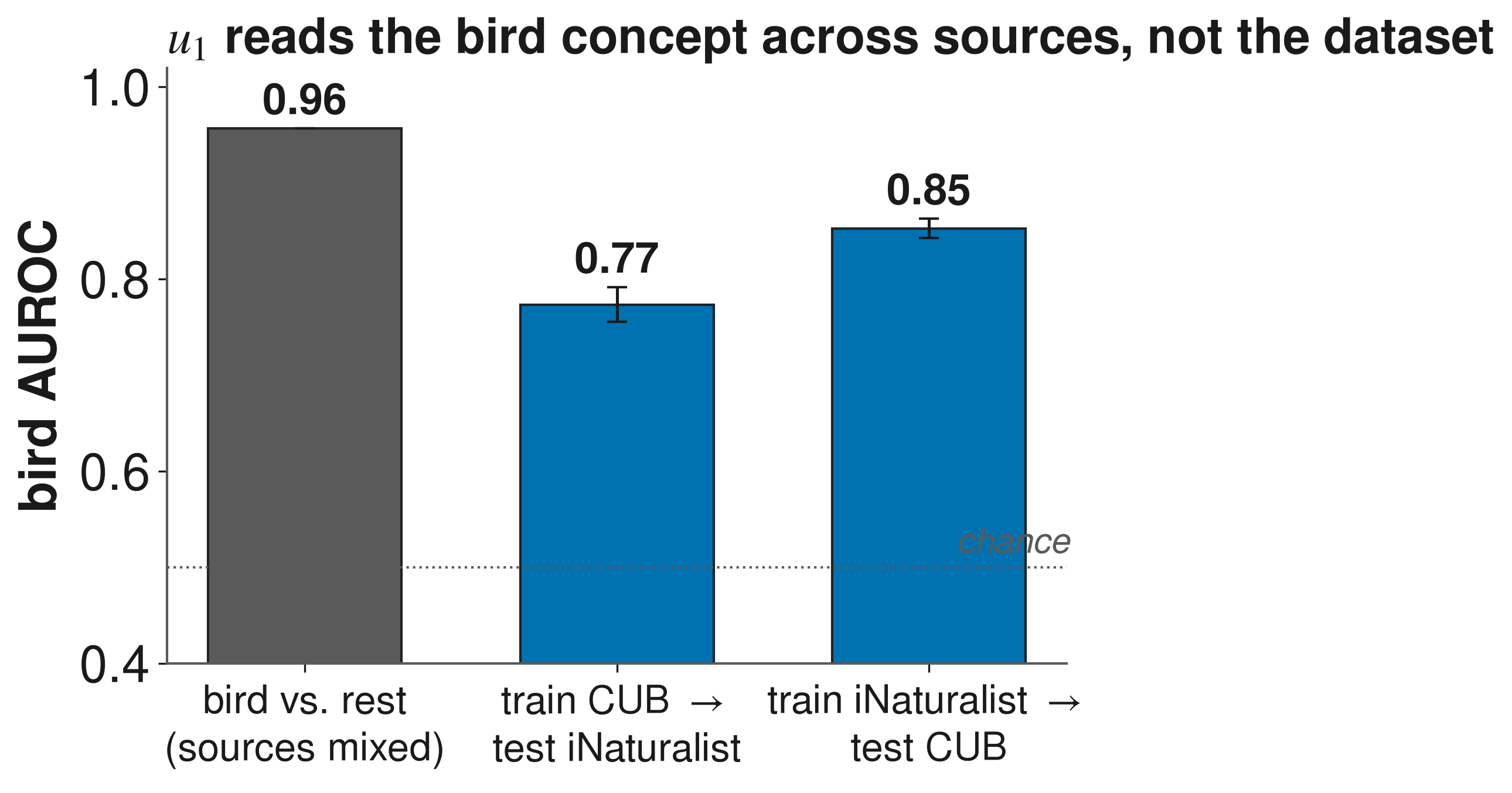}
    \caption{A bird-vs-rest classifier on $u_1$ separates birds from the other $757$ SD-1.5 adapters at $0.957$ AUROC with the two sources (CUB-200 and the iNaturalist \emph{melanerpes} genus) pooled (left), and still recovers the held-out source's birds when trained on only one source (right; CUB$\to$iNat $0.77$, iNat$\to$CUB $0.85$), so $u_1$ encodes the bird subject, not merely dataset provenance. Bars are mean over 5-fold CV, error bars $\pm1$ s.d., dotted line chance.}
    \label{fig:birds}
\end{figure}

\section{Routing and Complementarity Details}
\label{app:routing}

The weight-space ($u_1$) and activation (Gaussian-probe) views are computed for the same adapters; we align them by adapter key and evaluate every method under one shared 5-fold cross-validation. \Cref{tab:routing-full} gives the full results. Average fusion averages the two classifiers' class probabilities, and the confidence-gated router predicts, per adapter, with whichever view assigns the higher maximum softmax probability. The per-category oracle routes each category to its better single view. Both fusion and the confidence router \emph{exceed} it, because they combine the views per adapter rather than per category. The router selects $u_1$ on $52\%$ of adapters, indicating the two views carry comparable, partly independent signal. \Cref{fig:routing-perclass} shows that average fusion sits at or above the better view on every category. We emphasize this is an analysis of complementarity, not a deployed pipeline.

\begin{table}[t]\centering\small
\begin{tabular}{@{}lccl@{}}
\toprule
Method & AUROC & BA & label-free? \\
\midrule
$u_1$ (weights)       & 0.975 & 0.743 & yes, no GPU \\
Gaussian (activations) & 0.986 & 0.816 & yes, GPU \\
confidence-router      & 0.993 & 0.901 & yes \\
fusion (concat)        & 0.993 & 0.879 & yes \\
\textbf{fusion (avg)}  & \textbf{0.995} & \textbf{0.900} & yes \\
\midrule
per-category oracle    & 0.988 & --- & upper bound \\
\bottomrule
\end{tabular}
\caption{Routing/complementarity on the $860$-adapter paired set (shared 5-fold CV). Oracle routes per category: $u_1$ for age, genus, land cover, objects; Gaussian for birds, scenes, textures. Both fusion and the label-free confidence router beat either view and exceed the per-category oracle bound.}
\label{tab:routing-full}
\end{table}

\Cref{fig:baseline} compares $u_1$ and Gaussian probing on the age and genus classification tasks. \Cref{fig:complementarity} and \Cref{fig:routing} give the per-category and macro views of the two combined.

\begin{figure}[t]
\centering
\includegraphics[width=0.92\columnwidth]{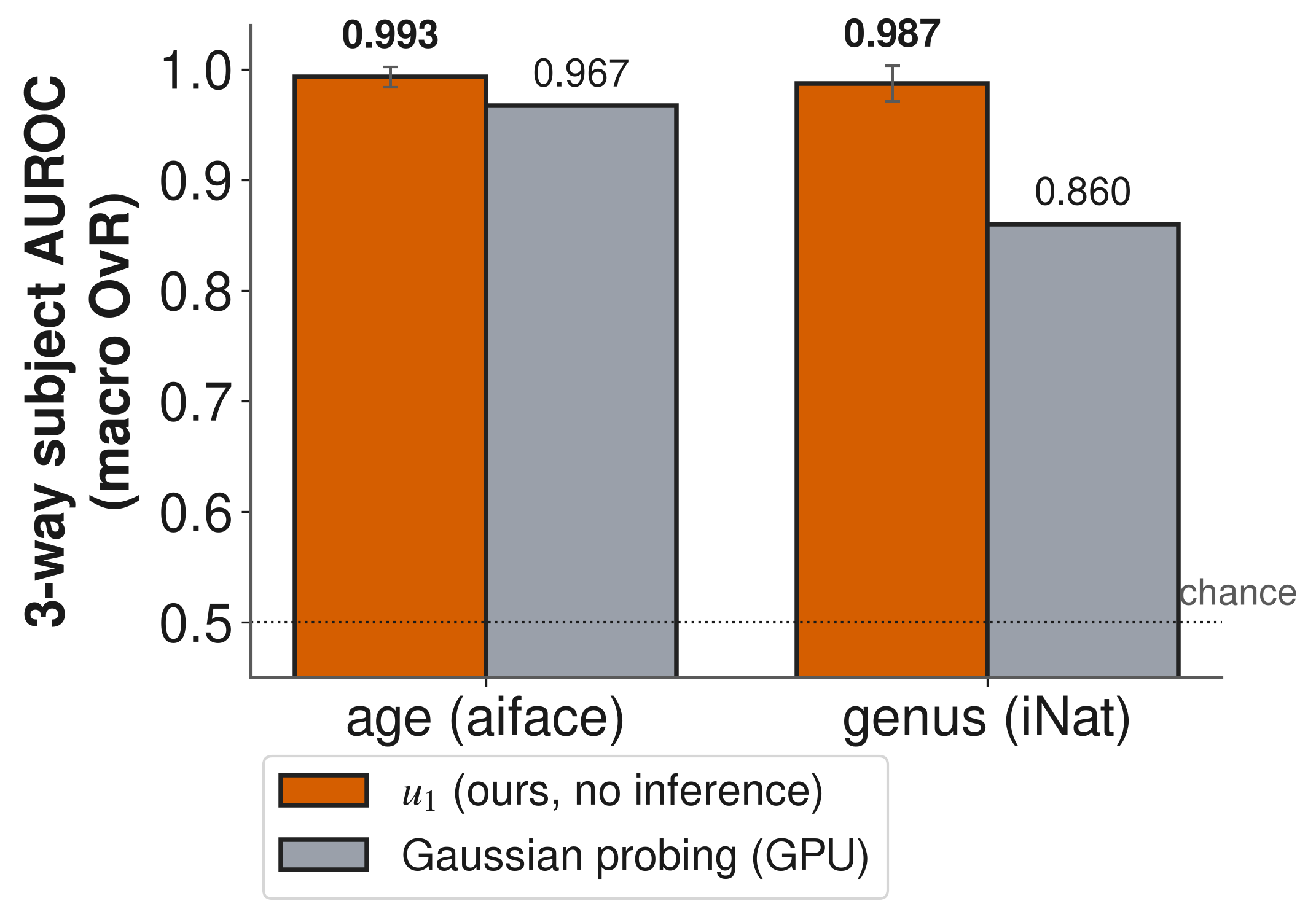}
\caption{Weight-space $u_1$ versus activation Gaussian probing on two controlled zoos. $u_1$ is higher on both zoos and needs no inference. Error bars are CV standard deviation.}
\label{fig:baseline}
\end{figure}
\begin{figure}[t]
\centering
\includegraphics[width=0.92\columnwidth]{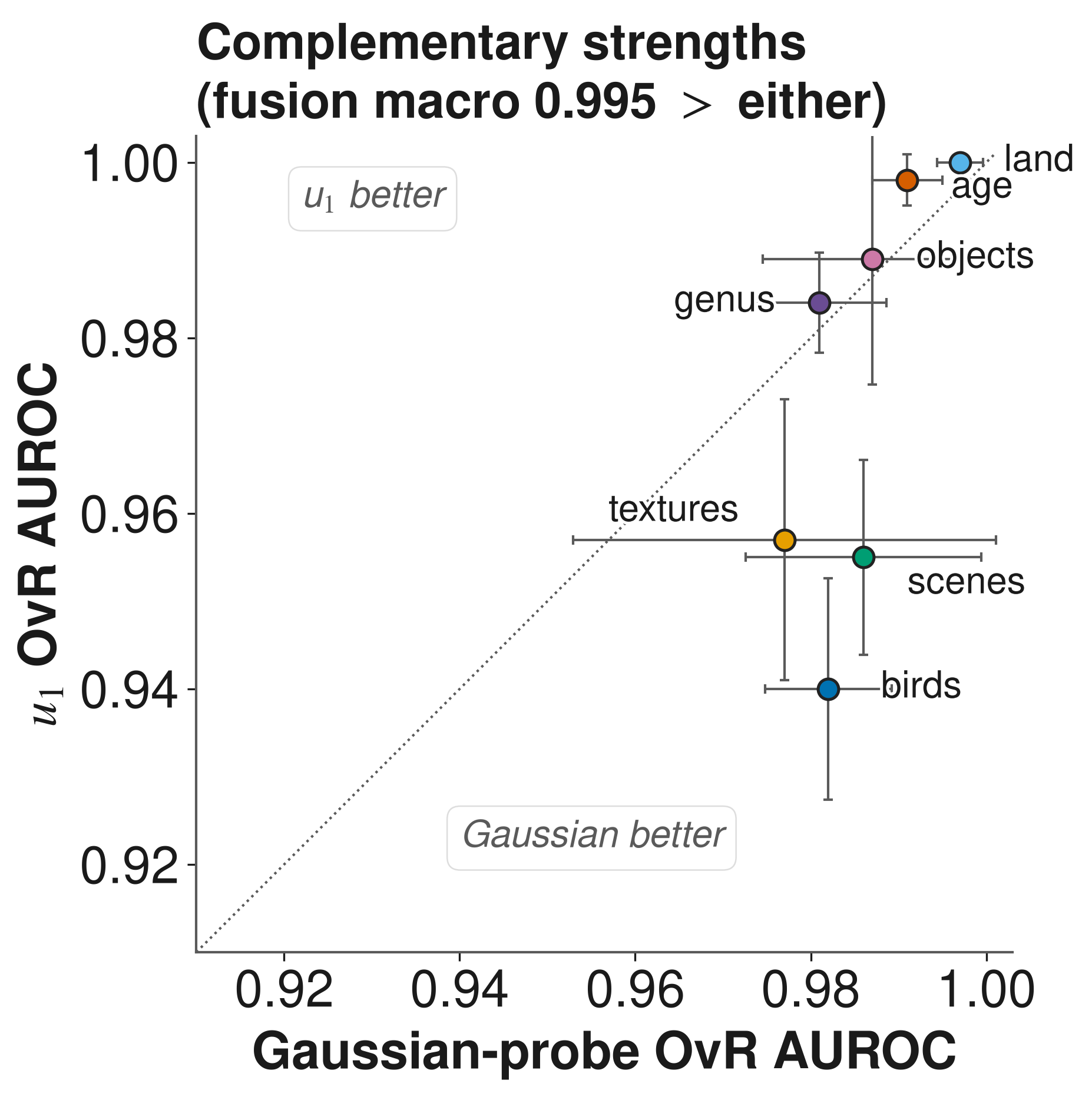}
\caption{Per-category one-vs-rest AUROC, $u_1$ vs Gaussian probing, on the seven-category benchmark. The two are complementary---a concatenation/average fusion reaches 0.993--0.995 macro AUROC, above either method alone.}
\label{fig:complementarity}
\end{figure}
\begin{figure}[t]
\centering
\includegraphics[width=\columnwidth]{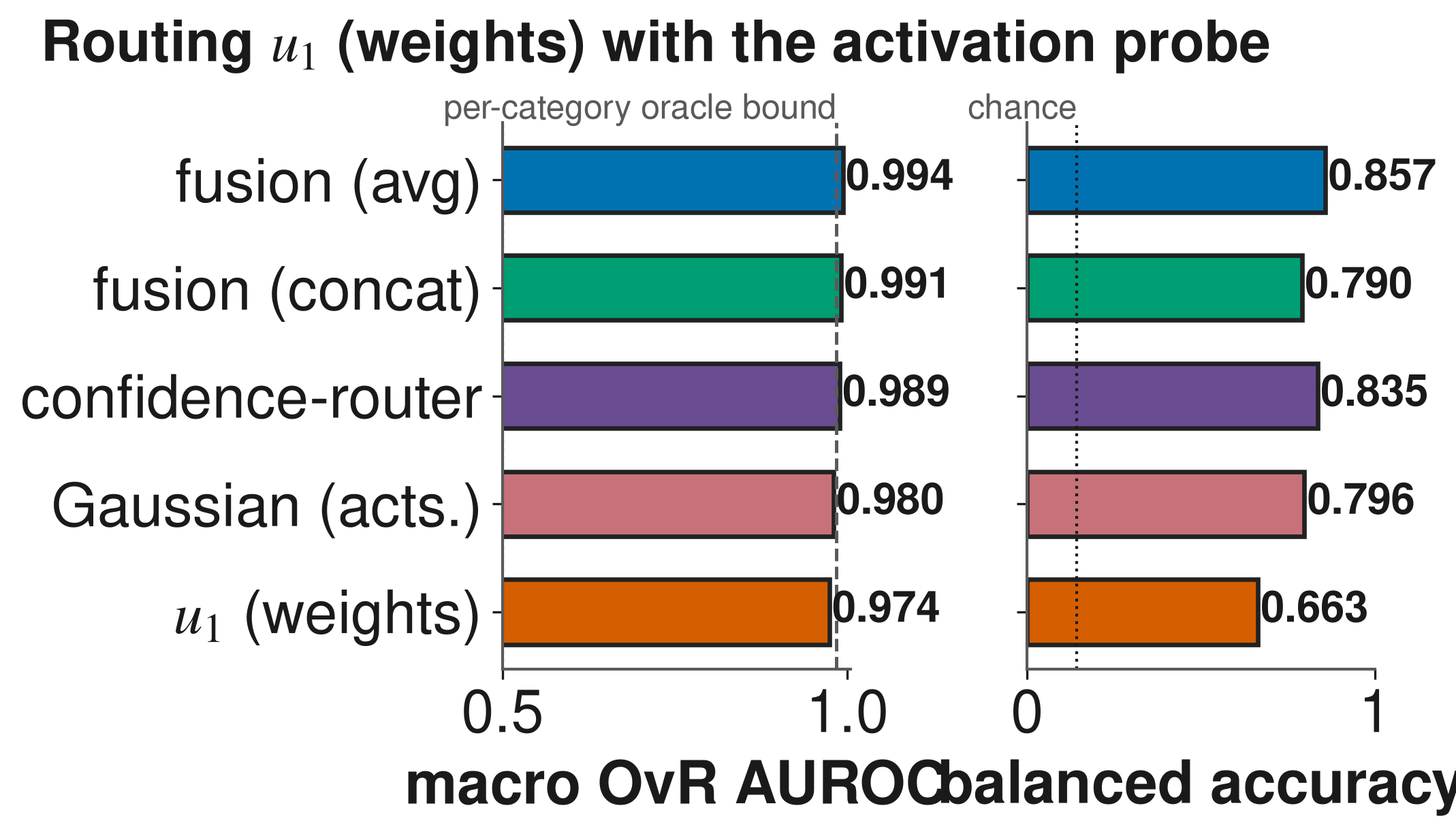}
\caption{Routing $u_1$ (weights) with the Gaussian activation probe on the paired adapter set (macro OvR AUROC and balanced accuracy, shared 5-fold CV).}
\label{fig:routing}
\end{figure}

\begin{figure}[t]
\centering
\includegraphics[width=\columnwidth]{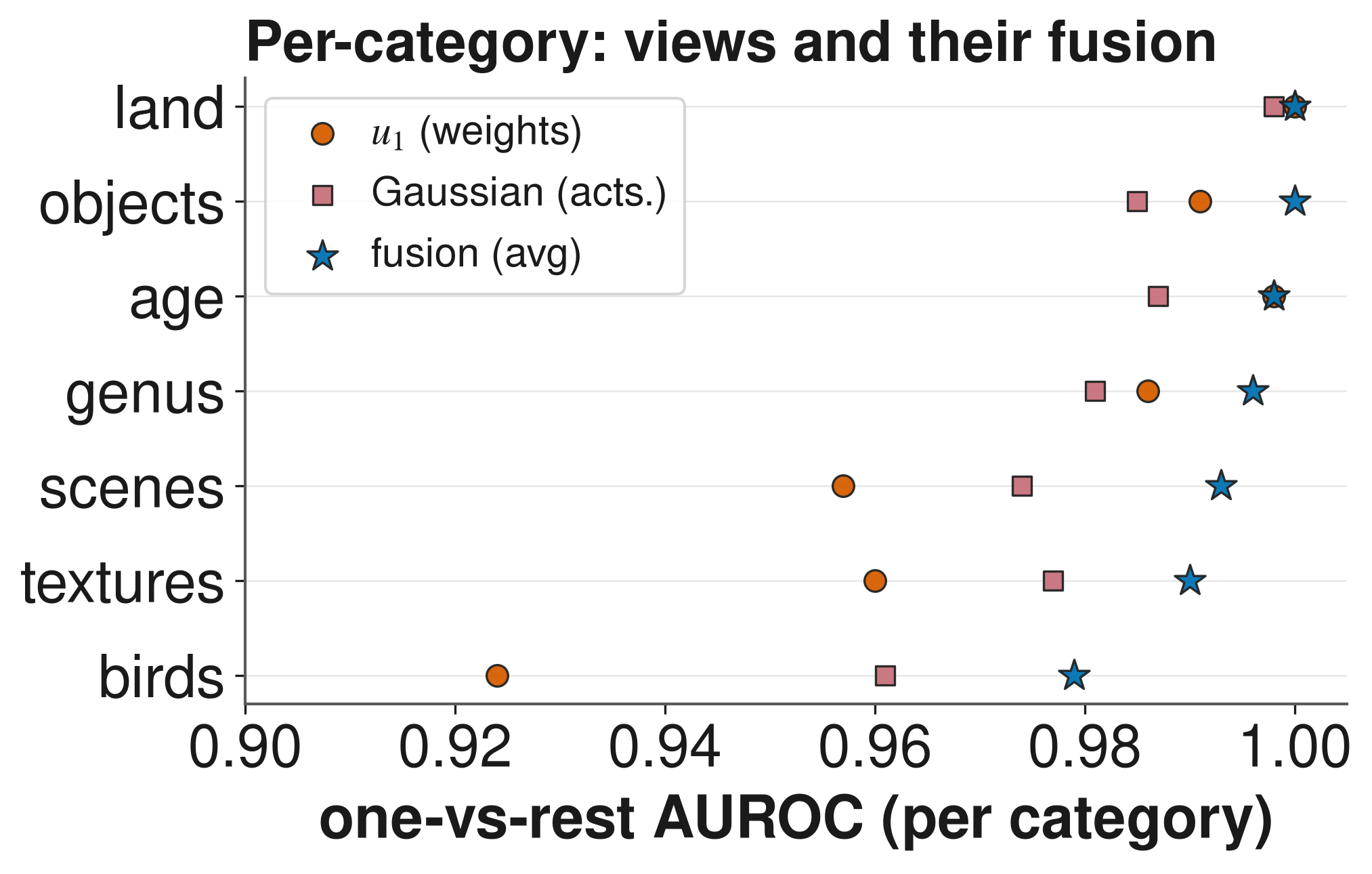}
\caption{Per-category one-vs-rest AUROC for the two views and their average fusion on the paired set. Fusion sits at or above the better view on every category; the views trade off (e.g.\ $u_1$ stronger on age/objects, Gaussian on birds/scenes/textures).}
\label{fig:routing-perclass}
\end{figure}

\end{document}